\documentclass[journal]{IEEEtran}
\usepackage{times}

\usepackage[numbers]{natbib}
\usepackage{multicol}
\usepackage[bookmarks=true]{hyperref}
\usepackage[english]{babel}
\usepackage{amsmath}
\usepackage{amssymb}
\usepackage{amsthm}
\usepackage{bbm}
\usepackage{optidef}
\usepackage{nccmath} 
\usepackage{xcolor}
\usepackage[ruled,vlined,linesnumbered]{algorithm2e}
\SetAlgoSkip{}
\usepackage{algpseudocode}
\usepackage[inkscapelatex=false]{svg}
\usepackage{graphicx}
\usepackage[caption=false]{subfig}
\usepackage[export]{adjustbox}
\usepackage{comment}

\newcommand{\VaR}{\textup{VaR}}
\newcommand{\CVaR}{\textup{CVaR}}
\newcommand{\E}{\mathbb{E}}
\newcommand{\CDF}{\textup{CDF}}
\newcommand{\ub}{\textup{ub}}

\newcommand{\Bin}{\textup{Bin}}
\newcommand{\X}{\mathcal{X}}
\newcommand{\U}{\mathcal{U}}
\renewcommand{\P}{\mathcal{P}}

\DeclareMathOperator*{\argmin}{arg\,min}

\newtheorem{theorem}{Theorem}
\newtheorem{corollary}{Corollary}
\newtheorem{lemma}{Lemma}
\newtheorem{definition}{Definition}
\newtheorem*{remark}{Remark}

\newtheorem{assumption}{Assumption}

\begin{document}

% paper title
%\title{Finite-Sample Performance Bounds on Control Plans for Stochastic Robotic Systems}

\title{Guarantees on Robot System Performance Using Stochastic Simulation Rollouts}

\author{Joseph A. Vincent$^{1*}$, Aaron O. Feldman$^{1*}$, and Mac Schwager$^{1}$% <-this % stops a space
\thanks{The NASA University Leadership initiative (grant \#80NSSC20M0163) provided funds to assist the authors with their research, but this article solely reflects the opinions and conclusions of its authors and not any NASA entity. This work was partly supported by ONR grant  N00014-23-1-2354.  The first author was also supported on a Dwight D. Eisenhower Transportation Fellowship. The second author was also supported by a National Science Foundation Graduate Research Fellowship under Grant No. 2146755.}% <-this % stops a space
\thanks{$^{1}$Department of Aeronautics and Astronautics, Stanford University, Stanford, CA 94305, USA, {\texttt\footnotesize \{josephav, aofeldma, schwager\}@stanford.edu}}%
\thanks{$^*$ denotes equal author contribution}%
\thanks{The code associated with this work can be found at \href{https://github.com/StanfordMSL/performance_guarantees}{\texttt{https://github.com/StanfordMSL/performance\_guarantees}}}%
}

\maketitle

%%%%% Abstract and Key Words %%%%%%%%%%
\begin{abstract}
We provide finite-sample performance guarantees for control policies executed on stochastic robotic systems. Given an open- or closed-loop policy and a finite set of trajectory rollouts under the policy, we bound the expected value, value-at-risk, and conditional-value-at-risk of the trajectory cost, and the probability of failure in a sparse cost setting. The bounds hold, with user-specified probability, for any policy synthesis technique and can be seen as a post-design safety certification. Generating the bounds only requires sampling simulation rollouts, without assumptions on the distribution or complexity of the underlying stochastic system. We adapt these bounds to also give a constraint satisfaction test to verify safety of the robot system. \textcolor{black}{We provide a thorough analysis of the bound sensitivity to sim-to-real distribution shifts and provide results for constructing robust bounds that can tolerate some specified amount of distribution shift.} Furthermore, we extend our method to apply when selecting the best policy from a set of candidates, requiring a multi-hypothesis correction. We show the statistical validity of our bounds in the Ant, Half-cheetah, and Swimmer MuJoCo environments and demonstrate our constraint satisfaction test with the Ant. Finally, using the 20 degree-of-freedom MuJoCo Shadow Hand, we show the necessity of the multi-hypothesis correction.
\end{abstract}

% the expected value, value-at-risk (VaR), and conditional-value-at-risk (CVaR) of the cost under the plan. We can similarly bound the probability of constraint satisfaction, which allows us to assess whether the plan satisfies given chance constraints. 
% The bounds are guaranteed to hold, with a user-specified probability, without knowledge of the underlying probability distributions driving the black-box system. We only require a black-box simulator to sample stochastic trajectory rollouts.
% We demonstrate the validity of our method on several MuJoCo environments: the Ant, Half-cheetah, and Swimmer, and for in-hand robot manipulation using the 20 degree-of-freedom Shadow Hand. 
% The bounds are independent of plan dimensionality, making them suitable for complex systems, e.g., a robot hand. 

%%%%%%%%%%%%%%%%%%%%%%%%%%%%%%%%%%%%%%%%%%%%%%%%

\IEEEpeerreviewmaketitle

%%%%% Introduction %%%%%%%%%%

\section{Introduction} \label{sec:Intro}
% Put intro here.
% This demo file is intended to serve as a ``starter file" for the
% Robotics: Science and Systems conference papers produced under \LaTeX\
% using IEEEtran.cls version 1.7a and later. 

It is essential that robots be able to operate safely and successfully under diverse sources of uncertainty, including uncertainty about their own dynamics and state, friction and contact forces, the future motion of other agents, and environment geometry. For example, robot manipulators must interact with objects having uncertain geometries or physical parameters, legged robots must locomote on uncertain terrain, and autonomous vehicles must avoid colliding with other agents whose future trajectories are uncertain. While it is common to have a simulation model reflecting these diverse sources of uncertainty, we rarely have access to closed-form mathematical models, making it difficult to provide rigorous performance and safety guarantees. We address this challenge, presenting statistical performance bounds and safety tests for arbitrary robotic systems given a finite set of trajectory samples from a stochastic simulator.

\begin{figure}
    \centering
    % \includesvg[width=\linewidth]{figs/single_plan_overview.svg}
    \includegraphics[width=\linewidth]{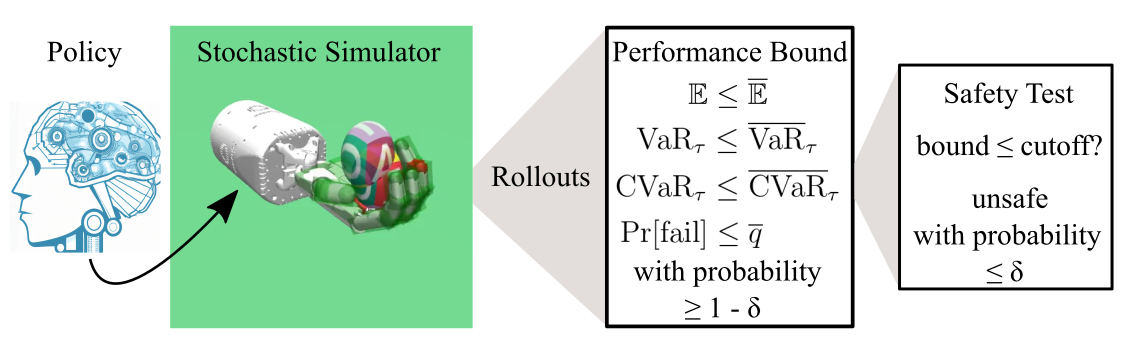}
    \caption{Overview of our method for bounding performance for a single policy using a stochastic simulator. The policy is executed in simulation $n$ times to collect trajectory rollouts. The cost or constraint function is evaluated for each rollout and these samples are used to form a distribution-free upper bound on a given performance measure (expected value, value-at-risk, conditional-value-at-risk, or probability of failure) that is guaranteed to hold with probability at least $1-\delta$. These probabilistic bounds may also be used to ensure safety by testing constraint satisfaction for the performance measures. The finite-sample bound guarantee ensures that these tests incorrectly accept a policy as safe with at most $\delta$ probability. We demonstrate this pipeline for several MuJoCo environments and extend the method to compare multiple policies for manipulating an egg of uncertain mass and friction.}
    \label{fig:single_plan_selection}
\end{figure}

Performance for stochastic robotic systems is typically quantified with an expected trajectory cost, or with risk-sensitive performance measures such as Value-at-Risk (VaR) or Conditional Value-at-Risk (CVaR). \textcolor{black}{These risk sensitive measures have been widely used in the optimization of financial portfolios~\cite{rockafellar2000optimization}, and are more recently being adopted for use in robotic control problems~\cite{ruszczynski2010risk,pavone}.} Performance can also be quantified by the probability of task success (e.g., probability that an object is not dropped, or a robot does not fall). Symmetrically, safety is often enforced by putting constraints on these performance measures: expected value constraints, VaR constraints, CVaR constraints, or constraints on success probability. We provide probabilistic bounds for all of these performance measures.

Most existing approaches for computing or bounding expected value, VaR, CVaR, or success probability (i) assume a known distribution for the uncertainty (e.g., Gaussian), (ii) use a large number of simulation rollouts to approximate uncertainty without formal guarantees, or (iii) provide formal guarantees that hold asymptotically as the number of samples approaches infinity. In contrast, our methods are both distribution-free (they require no knowledge of the underlying probability distribution), and finite-sample (they hold with a finite number of samples, not just asymptotically). A core insight is that the trajectory cost of the robotic system can be treated as a scalar random variable, and a stochastic simulator can be viewed as an elaborate random number generator, producing independent and identically distributed (IID) samples of the trajectory cost. We can therefore apply distribution-free, finite-sample statistical analysis tools. We only require two key assumptions: 
\begin{enumerate}
\item \emph{\textcolor{black}{The simulative dynamics model exactly represents the stochastic dynamics encountered during execution, in addition to modeling the roboticist's uncertainty in initial conditions and simulation parameters}}, and 
\item\emph{Successive simulations are IID, that is, there is no memory or distributional shift between trajectory rollouts.}
\end{enumerate}
\textcolor{black}{For the first assumption, there are three uncertainties the simulator may capture, uncertainty over (i) state transitions, (ii) initial conditions, and (iii) simulation parameters. We assume these uncertainties are modeled as random variables from some distribution as is common in other robotics problems such as state estimation/filtering.} 

We derive simple formulas using foundational statistical principles to compute upper bounds on expected value, VaR, and CVaR, as well as an upper bound on the probability of failure in a binary success/failure reward model. These bounds are probabilistic, that is, they are allowed to be wrong $\delta$ proportion of the time, for a user-specified error rate $\delta$. We further adapt these bounds to give constraint satisfaction tests for expected value, VaR, CVaR, and failure probability, with a guaranteed user-specified false positive rate (declaring the system safe, when it is actually unsafe). \textcolor{black}{The probability of the bound being valid ($1-\delta$) is often referred to as the \textit{confidence} or \textit{coverage} level of the bound. In Section~\ref{sec: bound_sensitivity} we investigate how the confidence level for each bound changes when there is sim-to-real distribution shift and Assumption 1 is not met. In Section~\ref{sec: robust_bounds} we detail how to retain a desired confidence level in light of such distribution shifts.} 

An overview of our method for bounding performance and testing constraint satisfaction for a policy is shown in Fig.~\ref{fig:single_plan_selection}. We also modify the bounds so they can be used to compare performance among multiple policies, as outlined in Fig.~\ref{fig:plan_selection}. Notably, we show that a multi-hypothesis correction is required to retain statistical guarantees when comparing multiple policies. We empirically demonstrate the validity of our bounds and constraint satisfaction tests in several MuJoCo~\cite{mujoco} environments simulated in Gymnasium~\cite{gymnasium_robotics2023github}, and verify our policy comparison bounds in a 20 degree-of-freedom MuJoCo Shadow Hand simulation manipulating an object with uncertain mass and friction.

Our bounds are independent of the system's complexity or dimensionality. They only depend on the number of rollouts and the user-specified error rate. Therefore, our method is appropriate for complex simulation models of high degree-of-freedom robots, featuring, e.g., discontinuities from contact, uncertainties in friction or reaction forces, fluidic or finite-element simulations for soft robots and deformable objects, and aerodynamic simulations for aerial robots. The bounds also apply regardless of how the simulation is derived, applying also when the simulation itself is learned (e.g, generative world models or multi-agent trajectory forecasters). The ``simulation'' can also be an experimental setup in which a control policy is repeatedly executed on a physical robot. Furthermore, our bounds are valid for open- or closed-loop policies, as well as deterministic or stochastic policies.  

We emphasize that we do not present a new policy optimization or planning technique in this work, but rather present a statistical method for bounding the performance of a given control policy or comparing performance among a set of policies. Our methods can be used as a verification tool to bound the performance or certify the safety of a policy obtained from any upstream optimizer (e.g., a reinforcement learning or optimal control design technique or even a large language model task planner). % Using our performance comparison bounds, our method can also be used internally within a policy optimizer to select a policy with a verified performance guarantee (e.g., within the CEM optimizer~\cite{mannor2003cem}). 

In summary, our primary contributions are
\begin{enumerate}
    \item Given an open- or closed-loop control policy, we \textcolor{black}{apply} probabilistic upper bounds for expected value, VaR, CVaR, and probability of task failure, from a finite set of simulated trajectory rollouts. 
    \item Similarly, we obtain a probabilistic test to verify the satisfaction of constraints on expected value, VaR, CVaR, or probability of failure. The test has a user-specified false acceptance rate.
    \item We describe a necessary multi-hypothesis correction to these performance bounds in the case of choosing the best among a finite set of candidate policies.
    \item \textcolor{black}{We provide analytic expressions for how the confidence of each bound changes as a function of sim-to-real distribution shift and we provide simple extensions to each bound to ensure a desired confidence level holds under a specified amount of sim-to-real distribution shift.}
\end{enumerate}

\textcolor{black}{We achieve the above for an arbitrarily complex simulator, learned or model-based, with diverse sources of uncertainty, and with any upstream policy generation or optimization approach. As evidence to this, we demonstrate our approach in several MuJoCo environments including with the 20 degree-of-freedom Shadow Hand.} The paper is organized as follows. We give related work in Section \ref{sec:Related}. In Section \ref{sec:Problem} we introduce our notation and define the problem setting. In Section \ref{sec: plan_evaluation} we present the distribution-free performance bounds and derive constraint satisfaction tests.  We empirically validate the bounds and constraint tests in MuJoCo simulations in Section \ref{sec: plan_eval_exp}. \textcolor{black}{In Section~\ref{sec: bound_sensitivity} we provide expressions for how the confidence level for each bound changes due to sim-to-real distribution shift and in Section~\ref{sec: robust_bounds} we detail how to retain a desired confidence level given some amount of distribution shift.} In Section \ref{sec: multi-hyp} we introduce the multi-hypothesis correction required when comparing bounds among multiple policies, and demonstrate the validity of this correction in Section \ref{sec:examples} in simulations with the MuJoCo Shadow Hand manipulating an uncertain object. We offer concluding remarks in Section \ref{sec:conclusion}, and give proofs of theorems and computational details of the simulation examples in the Appendix.

%%%%%%%%%%%%%%%%%%%%%%%%%%%%%%%%%%%%%%%%%%%%%%%%

%%%%% Related Work %%%%%%%%%%
\section{Related Work} \label{sec:Related}

\subsection{Sample-Based Performance Quantification}

We build on an emerging literature that uses finite samples from simulation rollouts to produce and optimize for distribution-free guarantees on system performance. % (developed in~\cite{caltech_verification},~\cite{caltech_rl},~\cite{caltech_sim2real},~\cite{caltech_policy_synth}, and~\cite{caltech_nonlinear})
Using results from randomized optimization, the authors of~\cite{caltech_verification} construct distribution-free bounds on the VaR of a given robustness metric for a robotic system. Their bound on the VaR is a special case of our bound given in Theorem \ref{thm: var_bound}. Similar methods are used in~\cite{caltech_rl} to verify safety and robustness of a reinforcement learning policy and in~\cite{caltech_sim2real} to place probabilistic bounds on the error of a simulated model. 

The authors of~\cite{caltech_policy_synth} use distribution-free statistics to place bounds on coherent risk measures (such as $\textup{CVaR}$) and on the quality of optimization via random search. They use these results to formulate a method that randomly searches over policies and chooses the one with the least upper bound on risk. In~\cite{caltech_nonlinear}, the same bounds are used to find a nonlinear control plan that is better than a specified percentage of plans. As we show in this work, when policy cost is not deterministic, optimizing for the sample-based bound requires a multi-hypothesis correction to retain validity. %Such analysis is absent from~\cite{caltech_policy_synth}, making their selected policy unlikely to obtain the bound they optimized for. Mac: we don't want to pick fights.  These authors are likely to be out reviewers. :)
This is a crucial step in the policy synthesis process that our paper addresses in Section \ref{sec: multi-hyp}.

The authors of~\cite{risk_verification} perform verification of closed-loop stochastic systems. Like us, they take a distribution-free finite-sample approach, but with some key differences. While they are primarily interested in verifying closed loop systems with neural network controllers, we take a broader view to systems which need not be differentiable, continuous, or defined in closed form. \cite{risk_verification} also discusses choosing the least risky controller from a set of controllers but fails to note the need for a multi-hypothesis correction. Lastly, we use a tighter VaR bound not based on the Dvoretzky–Kiefer–Wolfowitz (DKW) inequality \cite{dvoretzky1956asymptotic, massart1990tight} and take a different approach to constraints; we provide analysis for handling a variety of risk-sensitive constraints whereas~\cite{risk_verification} focuses on signal-temporal-logic (STL) constraints. \textcolor{black}{Lastly, none of these works address how bound confidence changes with distribution shift.}

%While this percentage optimality guarantee is valid when the cost for a control sequence is deterministic, it fails when the environment (and hence cost) is uncertain, as we show in this work. 
%They also use these techniques to place bounds on optimization problems, determining how many uniform random samples of a decision variable need to be taken in order to ensure (with high probability) that the best one is better than some desired fraction of the feasible set. 

%Lastly, our experiments showcase our approach on systems with more complicated dynamics that make and break contact with their environment, resulting in more complicated stochasticity. as well as more complicated stochasticity which includes a mixture of discrete and continuous random variables.
% Joe regarding above: Can we still argue that we have a mixed distribution if no class uncertainty? I think the discrete contact modes still give rise to mixed distribution, right? 

\subsection{Risk-Sensitive Control}

Our method can be used to certify control policies obtained from existing risk-sensitive control design techniques. We classify the approaches for risk-sensitive control into three broad categories: parametric, distributionally robust, and sampling-based. 

In the parametric category are works imposing distributional and structural assumptions to efficiently quantify risk. The authors of~\cite{pavone} compute risk averse policies (in the sense of CVaR) for finite state and action Markov Decision Processes (MDPs) by solving a surrogate MDP. The authors of~\cite{cvar_barrier} synthesize a CVaR-safe controller for linear systems using barrier functions. \cite{lew_chance} enforces obstacle avoidance chance constraints assuming a Gaussian dynamics disturbance. While parametric approaches provide efficient mechanisms for quantifying and mitigating risk, we avoid the associated assumptions (i.e, on the dynamics or uncertainty distribution), so that our approach can be applied as a general certification step for arbitrary, complex systems.

% While their method directly optimizes without need for simulation, they assume finite state and action spaces and known transition probabilities. 
% assume a finite disturbance space with known probability mass function and exploit this to 
% In contrast with these methods, our approach allows for both discrete and continuous random variables and makes no assumptions on linearity of dynamics. 
% Continuous disturbances are oftentimes assumed to be Gaussian for mathematical convenience. 
% We probabilistically enforce chance constraints using a distribution-free sample-based method with provable false acceptance probability.

Instead of assuming a particular uncertainty distribution, some work adopts a distributionally robust approach. The authors of~\cite{cvar_dro} and~\cite{data_driven} propose CVaR constrained control where the CVaR is first estimated empirically. Then, based on a known ambiguity set for the disturbance distribution (using the Wasserstein metric) about the empirical CVaR, distributionally robust CVaR constraints are enforced at runtime. The authors of~\cite{entropic} add a constraint on the entropic value at risk (EVaR) in their MPC formulation using its dual representation as the worst-case expectation within a KL divergence-based ambiguity set. \textcolor{black}{In contrast with these works focusing on distribution mismatch, we focus on the setting where our simulator provides samples from the true uncertainty distribution, but our bounds do not require knowledge of that distribution. However, we also provide extensions in Section \ref{sec: robust_bounds} to construct robust bounds that hold even under simulator mismatch. Specifically, we use the simulator and a given robustness tolerance to implicitly define a distributional ambiguity set when constructing the robust bounds.} 

% Notably, the authors of~\cite{entropic} assume a discrete set of obstacle perturbations and directly use the associated probability mass function in their optimization. 

Sampling-based approaches use repeated draws of empirical performance to estimate the risk without imposing distributional assumptions. In~\cite{var_multi} each agent in a team estimates its VaR using the empirical quantile of recently observed rewards. The authors of~\cite{robust_options} impose a CVaR constraint during policy optimization by rewriting the CVaR as a tail expectation so that it can be approximated from rollouts. During MPC,~\cite{pets} and~\cite{particle_mpc} repeatedly rollout controls under the stochastic dynamics and optimize for the sequence minimizing the average associated trajectory cost.

\cite{lew2023sample} shows that, under certain conditions, the sample average solution becomes asymptotically optimal (as the number of drawn samples approaches infinity). However, in this work we are interested in producing finite-sample performance guarantees.~\cite{lew2023riskaverse} applies~\cite{lew2023sample} to CVaR-constrained trajectory optimization and provides a finite-sample bound on the CVaR constraint. However, they leverage concentration inequalities which hold uniformly across all controls. Our bounds hold pointwise, necessitating a multi-hypothesis correction, but when comparing only a modest number of policies this can yield tighter guarantees. In summary, while sampling-based control provides context for our work, we are interested in augmenting such methods with finite-sample, distribution-free statistical guarantees.

\subsection{Conformal Prediction}
Conformal prediction (CP) is increasingly being used for producing distribution-free guarantees in robotics. Given exchangeable data (a weaker condition than IID) and a scoring function, CP produces a confidence interval on the score of a new data sample~\cite{shafer_tutorial}. In this work, we adapt analysis tools from conformal prediction to obtain our VaR bound, and to assess chance constraint satisfaction. Similar to our work, several papers applying CP to robotics have viewed full robot trajectories as one sample. The authors of~\cite{asl_conformal} use collected data of unsafe trajectories to augment robotic fault-detection systems, achieving a guaranteed false negative rate. In~\cite{upenn_1}, the authors use CP to augment learned forecasting models (e.g., predicting pedestrian motion) with a confidence set of possible trajectories. In~\cite{upenn_2} the authors use adaptive conformal prediction to now adapt their confidence sets using online data. As in much of the CP robotics literature, we also view trajectories as the fundamental sample to get IID data. However, instead of using offline data, we evaluate performance using a generative stochastic simulator. %Thus, our approach also allows us to have varying amounts of uncertainty over time, by injecting varying noise in simulation. For instance, the distribution over object mass could tighten as the robot interacts with the object and infers its physical properties.

\subsection{Concentration Bounds}
Key to our approach is the use of sampling-based distribution-free concentration bounds for risk measures as this allows us to produce rigorous performance guarantees when planning with any arbitrarily complex simulator. 

Concentration bounds for $\textup{CVaR}$ were first given by~\cite{cvar_brown}, further refined by~\cite{cvar_wang}, and later improved upon by~\cite{cvar_thomas}. Each of these results requires bounds on the support of the random variable. Later,~\cite{cvar_kolla} provided concentration bounds in the case the random variable is sub-Gaussian or sub-exponential (weaker than boundedness by Hoeffding's Lemma~\cite{hoeffding}) which were improved upon in~\cite{cvar_bhat}. The CVaR bound we use in this paper is from~\cite{cvar_thomas}, but we express it in a simpler form and give an accompanying proof. We then extend this bound to give a novel bound on expected value. %and assume only bounded cost, which can be ensured as described in Section \ref{sec: plan_evaluation}. 
% to avoid assuming sub-Gaussian or sub-exponential distributions. We describe an approach to ensure bounded costs in Section \ref{sec: plan_evaluation}. 

% The results of~\cite{cvar_kolla} were improved upon in~\cite{cvar_bhat} by employing bounds on the Wasserstein distance between the empirical cumulative distribution function (CDF) and the true CDF. 

Unlike $\textup{CVaR}$ bounds, $\textup{VaR}$ bounds place no finite support or sub-Gaussianity restrictions. $\textup{VaR}$ bounds are often derived via the DKW inequality~\cite{cvar_kolla, var_dkw, var_sequential}. Although some authors claim these bounds as contributions to the field, an optimal VaR bound has been available since 2005~\cite{zielinski2005best} (although this optimal VaR bound is derived for continuous random variables, it can be extended for discontinuous random variables using the methods of \cite{scheffe1945non}). In this work we use a slightly suboptimal VaR bound because of its simple form and derivation which we present in the Appendix. The VaR bound we use is well-established in the statistics community, dating back to 1945 (see \cite{scheffe1945non, david2004order}), but its application to bounding policy performance is new. This classical bound is seemingly unknown to many practitioners and is strictly better than the bounds given in~\cite{cvar_kolla, var_dkw}.

% The bounds of~\cite{cvar_kolla} and~\cite{var_dkw} are slightly different, though both based on the DKW inequality. The authors of these bounds have noted that they may may be tightened using the local DKW results from~\cite{local_DKW}. In contrast,~\cite{var_sequential} presents a concentration sequence rather than a concentration bound. That is, as IID samples are obtained sequentially the bound remains valid. The price paid for this sequentially valid bound is that the width of the bound is roughly twice as wide as the other DKW bounds. In this paper we take a departure from these existing bounds and derive a novel bound for $\textup{VaR}$ which we conjecture is strictly tighter than the existing bounds based on the DKW inequality. A thorough comparison between our bound and the DKW-based bounds is given in the Appendix.

Concentration bounds on the expected value are more studied than for the $\textup{VaR}$ or $\textup{CVaR}$. Hoeffding's inequality is arguably the most influential bound of this sort~\cite{hoeffding} although the bounds given later by Anderson are no worse and often better~\cite{exp_anderson}. More modern work has focused on bounds which are not defined in closed form~\cite{thomas_exp}, and concentration sequences~\cite{howard2021time}. \textcolor{black}{The expectation bound we use in this paper is equivalent to the Anderson bound~\cite{exp_anderson}, but derived and expressed in a slightly different form.}

Lastly, bounds for the probability of success parameter in the Bernoulli distribution (often called binomial confidence intervals) have been studied extensively. Unlike other concentration bounds, the majority of methods for binomial confidence intervals are not guaranteed to hold with a user-defined probability (e.g., Wald, Wilson, Jeffreys', Agresti-Coull, etc. as described in ~\cite{brown2001interval, pires2008interval}). Methods guaranteed to meet the desired confidence level include those by Clopper and Pearson~\cite{clopper1934use}, Sterne~\cite{sterne1954}, Crow~\cite{crow1956}, Eudey~\cite{eudey1949, lehmann_textbook}, and Stevens~\cite{stevens1950}. Of these approaches, only the bounds from Eudey and Stevens return confidence intervals with exactly the desired coverage and they do so by inverting randomized hypothesis tests. We show that the confidence interval we derive in Theorem \ref{thm: feasibility_bound} is analogous to the one-sided Clopper-Pearson interval, which is known to be unimprovable amongst non-randomized approaches~\cite{wang_clopper}.

%%%%% Problem Setting %%%%%%%%%%
\section{Problem Setting} \label{sec:Problem}
Here we introduce notation and formalize the problem of quantifying performance and safety for a stochastic robotic system. For a given \textcolor{black}{fixed} time horizon $T$, we consider either an open-loop policy as a sequence of control actions $(U_0, \ldots, U_{T-1})$, or a closed-loop policy as a mapping (deterministic or stochastic) from state to action $U_t \sim \pi_t(X_t)$. We use the term policy to describe all of these cases, and cover all cases with the notation $\U$ to represent the stochastic sequence of control actions obtained by executing the policy in simulation. When executing the policy, the control actions drive state evolution from a random starting state $X_0$ via stochastic dynamics $F_t$
\begin{subequations}
    \begin{gather} \label{eq: IC}
        X_0 \sim \mathcal{X}_0, \\
        X_{t+1} \sim F_t(X_t, U_t) \ \ t = 0, ..., T-1.\label{eq: dynamics}
    \end{gather} \label{eq: sim_setup}
\end{subequations}

We do not assume access to an explicit functional form or distribution for $F_t$, but instead we assume we can sample $X_{t+1} \sim F_t(X_t, U_t)$ IID from a simulation of the system. We write the stochastic state trajectory as $\X = (X_0, \ldots, X_T)$, as with the control policy $\U$. The trajectory cost $J$ associated with a policy and a state sequence realization is composed of stage-wise costs $c_t$ as 
\begin{equation} \label{eq: cost}
    J(\X, \U) = c_T(X_T) + \sum_{t=0}^{T-1} c_t(X_t, U_t).
\end{equation}
As an important alternative case, we also consider the sparse reward setting often seen in reinforcement learning. We let $J = 1$ denote task failure and $J = 0$ denote task success, to keep the interpretation of $J$ as a cost rather than a reward.  

In addition to a cost function, we also consider a constraint function $g$ which may be used to impose trajectory constraints for safety (such as avoiding collisions, or avoiding control input limits), or for task success (such as not dropping an object, not falling down, or attaining a discrete goal). We consider the generated trajectory a success when it satisfies a given set of trajectory constraints,
    \begin{align} \label{eq: constr}
    g(\X, \U) \le 0.
\end{align}
The above constraint also applies in the binary success/failure setting, in which case $g = 0$ is task success and $g = 1$ is task failure, again with the convention of lower $g$ being safer. \textcolor{black}{When executing the policy, we assume that cost $J$ can still be defined even when the trajectory violates safety constraints and do not truncate or reject such trajectories. Rather, ``soft'' safety penalties may be incorporated by modifying the stage-wise costs $c_t$ in $J$, in addition to $g$ which separately encodes ``hard'' safety constraints.}

Though the policy is fixed, the associated cost $J$ and constraint function $g$ evaluate to random values due to the stochastic dynamics, which generate different state sequence realizations on different runs under the same policy. Therefore, to measure system performance, or to impose safety constraints, we must define a summarizing statistic over the cost or constraint functions, which we call a performance measure. 

\begin{figure}
    \centering
    % \includesvg[width=\linewidth]{figs/distribution_cartoon.svg}
    \includegraphics[width=\linewidth]{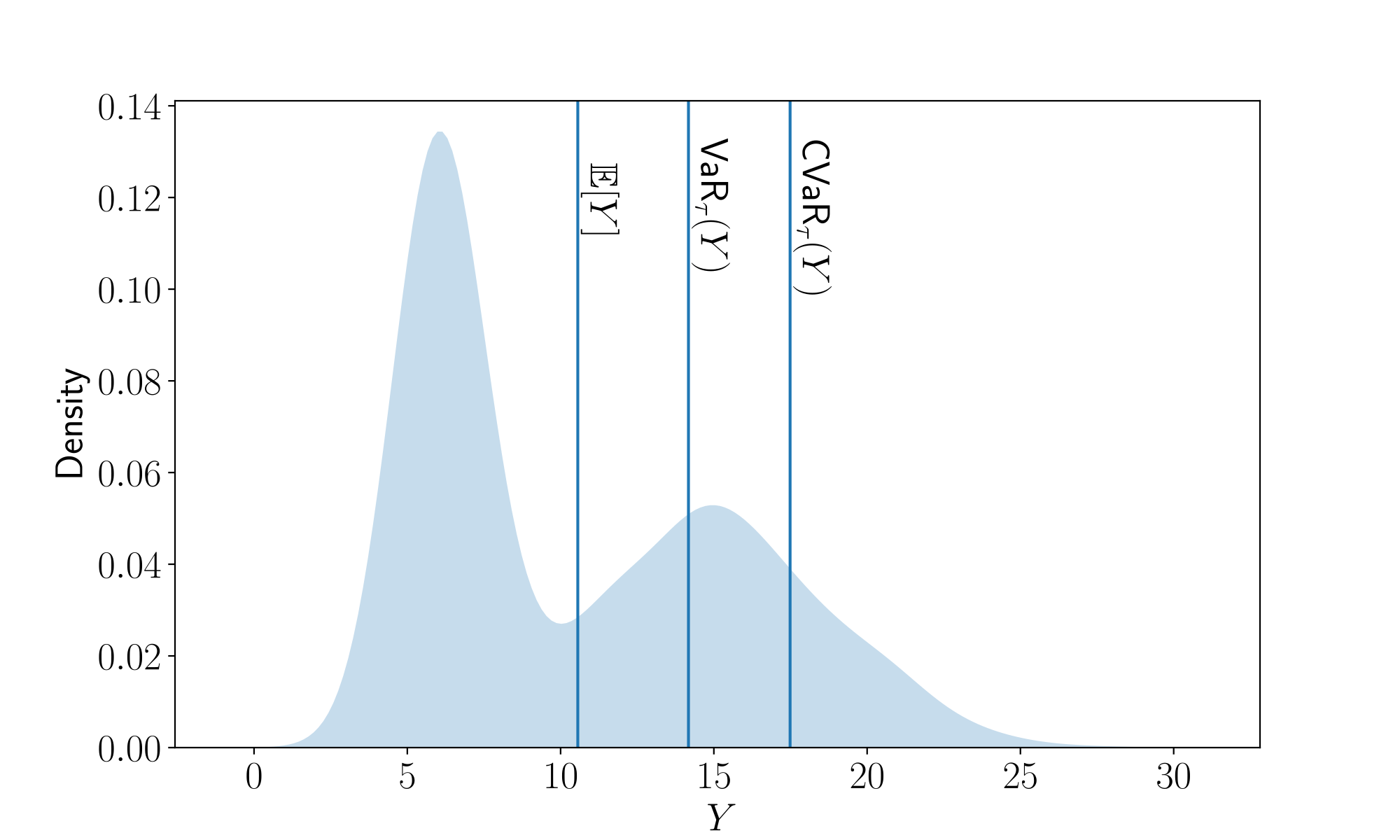}
    \caption{Visualization of the expected value, $\VaR_\tau$, and $\CVaR_\tau$ for an example distribution. Classic stochastic optimal control and reinforcement learning both seek to minimize the expected value of the total cost distribution. Risk-sensitive stochastic optimal control and reinforcement learning consider other measures of performance, such as VaR or CVaR of the cost.}
    \label{fig:dist_cartoon}
\end{figure}

\subsection{Performance Measures}
Consider a scalar random variable $Y$, representing either the value of the cost function $J$ or the constraint function $g$. We define a performance measure $\P(Y)$ as a summary statistic for the random variable $Y$. The most common choice of performance measure is the expected value $\mathcal{P}(Y) = \mathbb{E}[Y]$. However, to be more conservative, we often consider alternative measures that are risk-sensitive, such as the VaR or CVaR of the trajectory cost. $\VaR$, expected value, and $\CVaR$ are defined below, and visualized in Fig.~\ref{fig:dist_cartoon}. Alternatively, in the binary success/failure setting, the probability of failure is the natural performance measure to be minimized. In this work, we consider finite-sample methods for upper bounding all of these quantities.  

Recall the cumulative distribution function (CDF), which exists for any scalar random variable, whether continuous or not, is defined as $\CDF(y) \coloneqq \Pr[Y \le y]$.  We define $\VaR_\tau(Y)$ in terms of the CDF as

\begin{definition}[Value-at-Risk]
\label{Def:VaR}
    Given a scalar random variable $Y$, the value-at-risk of $Y$ at quantile $\tau \in (0,1)$ is
    \begin{align}
        \textup{VaR}_\tau (Y) \coloneqq \textup{inf}\{y \mid \CDF(y) \ge \tau \}.
    \end{align}
    If $Y$ has an invertible CDF then
    \begin{align}
        \VaR_\tau(Y) = \{y \mid \CDF(y) = \tau\} = \CDF^{-1}(\tau).
    \end{align}
\end{definition}

In plain words, $\textup{VaR}_\tau = y$ ensures that $\tau$ proportion of the probability mass of the random variable $Y$ is below the value $y$.  The $\textup{VaR}$ can be seen as a generalization of the inverse of the CDF and closely relates to the quantile as used in statistics.    

Recall the standard definition of the expected value of a continuous scalar random variable is $\E[Y] = \int_{-\infty}^{\infty}yp(y)\,dy$, where $p(y)$ is the probability density function. In fact, a more general definition of the expected value can be stated in terms of the $\VaR_\tau$ as follows,

\begin{definition}[Expected Value]
\label{Def:E}
    Given scalar random variable $Y$, the expected value of $y$ is
    \begin{align}
        \E[Y] \coloneqq \int_0^1\VaR_\tau(Y)\,d\tau.
    \end{align}
\end{definition}
The above definition applies to any scalar random variable (continuous, discrete, or mixed). In the case of a continuous random variable with invertible CDF, one can recover the standard definition through a change of variables\footnote{With the change of variables $\tau = \CDF(y)$, we have $d\tau = p(y)\,dy$ (recalling that $p(y) = dCDF(y)/dy$).  The domain of integration $\tau\in(0,1)$ becomes $y\in(-\infty, \infty)$, and since $\VaR_\tau(Y) = \CDF^{-1}(\tau)$, the integrand becomes $\CDF^{-1}(\CDF(y)) = y$, leading to $\int_{-\infty}^{\infty}yp(y)\, dy$.} $\tau = \CDF(y)$. 

One criticism of $\textup{VaR}$ as a risk-sensitive performance measure is that it ignores high-cost outcomes that may lie in the $1-\tau$ rightmost tail of the distribution. Conditional Value-at-Risk ($\textup{CVaR}$), defined below, addresses this concern.

\begin{definition}[Conditional Value-at-Risk]
\label{Def:CVaR}
    Given scalar random variable $Y$ and $\tau \in [0,1)$, the conditional value-at-risk of $Y$ at quantile $\tau$ is
    \begin{align}        
    \CVaR_\tau(Y) \coloneqq \frac{1}{1-\tau}\int_\tau^1\VaR_\gamma(Y)\,d\gamma.
    \end{align}
    If $Y$ has an invertible  CDF then
    \begin{align}
        \CVaR_\tau(Y) = \E[ Y \mid Y \ge \VaR_\tau(Y)].
    \end{align}
\end{definition}

There are several equivalent definitions for $\CVaR_\tau(Y)$. We prefer the one above as it applies to all scalar random variables (continuous, discrete, or mixed), and it highlights the clear relationship between $\VaR_\tau(Y)$ and $\E[Y]$. We can see that $\CVaR_\tau(Y)$ is simply the expected value of $Y$ taken over the top $1-\tau$ tail of the probability mass (and re-normalized by $1-\tau$ to ensure it remains a valid expectation). Taking $\tau = 0$ recovers the expected value. In contrast to $\VaR_\tau$, $\CVaR_\tau$ captures high-cost tail events. It therefore incorporates the worst case cost situations, and is often quite conservative. $\CVaR_\tau(Y)$ is always greater than or equal to both $\E[Y]$ and $\VaR_\tau(Y)$. However, $\E[Y]$ and $\VaR_\tau(Y)$ may lie in either order, depending on  $\tau$ and the distribution of the random variable $Y$. 

For simple distributions (e.g., Gaussians) $\E[Y]$, $\VaR_\tau(Y)$, and $\CVaR_\tau(Y)$ can be found. However, in more realistic robotics scenarios, distributions are non-Gaussian and usually are unknown, so none of these performance measures can be obtained in closed form. This motivates the need in this paper to give finite-sample bounds for these quantities.

We next formalize the probability of failure as a performance measure for problems with binary success/failure cost models.
\begin{definition}[Failure Probability]
    Given a Bernoulli random variable $Y$, we refer to $q = \Pr[Y = 1]$ as the failure probability.
\end{definition}
Note that failure probability is actually a special case of expected value, since $\Pr[Y = 1] = \E[Y]$ for a binary random variable $Y$. However, the binary case allows for significantly tighter bounds than the general bounds on expected value. We therefore treat this case separately.

\subsection{Cost Performance and Safety Performance}
The performance measures above can be applied to the cost function $\P(J)$ to measure cost performance, or to the constraint function $\P(g)$ to measure safety performance. In this paper, we present bounds $\overline{P}$ to probabilistically bound the cost performance from a finite set of simulation rollouts. Specifically, we give bounds of the form
\begin{equation}
\Pr[\P(J) \le \overline{P}] \ge 1-\delta,
\end{equation}
where $\delta$ is a user-defined error rate for the bound.

Similarly, safety is often quantified by putting constraints on a performance measure applied to the constraint function, $\P(g) \le C$, for some constant $C$. Specifically, we consider safety as a bound on any of the above performance measures,
\begin{subequations} \label{eq: g_constr}
    \begin{align}
        \E[g] &\leq C_{\E},  \label{eq: g_e} \\
        \VaR_\tau(g) &\leq C_{\VaR}, \label{eq: g_var} \\
        \CVaR_\tau(g) &\leq C_{\CVaR}, \label{eq: g_cvar} \\
        \Pr[g = 1] &\le C_{q}. \label{eq: g_p}
    \end{align}
\end{subequations}
where in (\ref{eq: g_p}) we consider the case where $g$ is assumed to be binary with $g = 1$ indicating failure.

The above formulation in (\ref{eq: g_constr}) already captures chance constraints (as seen in stochastic optimal control) which require that the trajectory constraint (\ref{eq: constr}) be satisfied with sufficiently high probability. This follows because constraining $\VaR_\tau(g)$ in (\ref{eq: g_var}) is equivalent to imposing a chance constraint on $g$,
\begin{equation}
    \VaR_\tau(g) \leq 0 \iff \Pr[g \leq 0] \ge \tau.
\end{equation}

In fact, chance constraints can also be modeled as a binary success/failure of the type (\ref{eq: g_p}), where $g = 1$ denotes the event that the trajectory fails the constraint (\ref{eq: constr}), and $g = 0$ for a successful trajectory satisfying the constraint.

Notice we cannot directly evaluate whether or not the constraints in (\ref{eq: g_constr}) hold, because we cannot compute their left hand side when we only have access to a simulator. We therefore define a test for whether $\overline{P} \le C$, using the finite-sample bound $\overline{P}$. If $\overline{P}$ passes this test, we declare the constraint satisfied. In the following section we prove that such a test can be constructed with a user-specified false positive error rate, only concluding an unsafe policy is safe some small fraction of the time.

%%%%%%%%%%%%%%%%%%%%%%%%%%%%%%%%%%%%%%%%

%%%%% Finite-Sample Performance Bounds %%%%%%%%%%
\section{Finite-Sample Performance Bounds} \label{sec: plan_evaluation}
In this section we provide finite-sample upper bounds for the expected value, VaR, CVaR, and failure probability. We also define constraint satisfaction tests based on these bounds. We require access to IID samples of the total cost $J$ or constraint function $g$ under the policy being evaluated, which we assume are obtained from \textcolor{black}{repeatedly executing the policy for a given time horizon $T$ in }a stochastic simulator of the robot system. Each bound presented holds probabilistically, with probability \textcolor{black}{at least} $1-\delta$, where the randomness stems from the bound itself being a function of a finite set of random samples. In fact, if no distributional assumptions are made, one can only formulate bounds that hold probabilistically (see Section 5 of~\cite{vovk_conditional}). We refer to $\delta$ as the user-specified error rate for the bound \textcolor{black}{and to the guaranteed probability $1 - \delta$ that the bound holds as the confidence level of the bound}. All proofs are deferred to the Appendix. An overview of the method is shown in Fig.~\ref{fig:single_plan_selection}. 

As explained previously, our results rest upon two foundational assumptions.
\begin{assumption}[Accurate Simulation Model]
\label{Ass:AccurateSim}
\textcolor{black}{The simulative dynamics model in (\ref{eq: sim_setup}) exactly represents the stochastic dynamics encountered during execution, in addition to modeling the roboticist's uncertainty in initial conditions and simulation parameters.}
\end{assumption}

\begin{assumption}[IID]
\label{Ass:IID}
Successive simulations are independent and identically distributed, that is, there is no memory or distributional shift between trajectory rollouts.
\end{assumption}
Of course, these assumptions will never exactly hold in practice, as there is always some sim-to-real gap. However, qualitatively, the closer the simulation model is to the real robotic system, the more reliable the bounds will be. \textcolor{black}{ In Section \ref{sec: bound_sensitivity} we address the sensitivity of the bounds to some sim-to-real gap, and in Section \ref{sec: robust_bounds} we show how to construct bounds which are robust to this gap.}

\subsection{Performance Bounds}
The bounds make use of the concept of order statistics, defined as follows.

\begin{definition}[Order Statistics] For a set of $n$ samples $J_{1:n}$ drawn IID, we let $J_{(k)}$ denote the $k$th order statistic, obtained by arranging the samples in order from smallest to largest and taking the $k$th element in the sequence.
\end{definition}

\begin{definition}[Binomial Distribution]
Let $\textup{Bin}(k;m,p)$ denote the Binomial cumulative distribution function, with $m$ trials, success probability $p$, evaluated at $k$ successes. 
\end{definition}

\begin{theorem}[VaR Bound] \label{thm: var_bound}
    Consider $\tau, \delta \in (0,1)$ and $n$ IID cost samples $J_{1:n}$, and let $k$ be the smallest index such that $\Bin(k-1;n,\tau) \ge 1 - \delta$. We have the following probabilistic upper bound on $\VaR_\tau(J)$,
    \begin{align}
        \overline{\VaR}_\tau \coloneqq J_{(k)},
    \end{align}
    which has the property
    \[
    \Pr[ \VaR_\tau(J) \le \overline{\VaR}_\tau ] \ge 1-\delta.
    \]
    A feasible value for $k$ exists when $n \geq \lceil \ln(\delta) / \ln(\tau) \rceil$, i.e., $n$ is large enough to ensure $\textup{Bin}(n-1;n,\tau) \ge 1 - \delta$.
\end{theorem}

% \begin{proof}
%     The proof for this as well as for lower bounds and two-sided bounds are given in the Appendix.
% \end{proof}

The $\VaR$ bound above simply chooses one of the order statistics based on a test involving the cumulative binomial distribution. Its proof (in the Appendix) is inspired by similar analyses in Conformal Prediction~\cite{shafer_tutorial, gentle_intro}. In our experience, this bound is considerably tighter than other $\VaR$ bounds in the recent literature (e.g.,~\cite{cvar_kolla, var_dkw}), and tends to be tighter in practice than the $\CVaR$ and $\E$ bounds below. As previously stated, the form of this bound is well known in the statistics literature~\cite{scheffe1945non, david2004order}, but its application to bounding policy performance is new. The $\VaR$ bound in~\cite{caltech_policy_synth} is a special case of this one, which considers the bound arising from the largest order statistic, $J_{(n)}$. 

Before stating the $\CVaR$ and $\E$ bounds, we require an additional assumption.

\begin{assumption}[Almost Sure Upper Bound]
\label{Ass:UpperBound}
    We have an almost sure upper bound $J_\ub$ such that $\Pr[J \le J_\ub] = 1$.
\end{assumption}

It may seem circular to require one upper bound in order to produce another upper bound.  The idea is to combine the order statistics $J_{(i)}$ with the almost sure upper bound to produce a significantly tighter bound. In fact, any finite sample upper bound on $\CVaR$ or $\E$ requires knowledge of such an {\it a priori} known bound on the right tail of the distribution of $J$. Two common choices are an almost sure upper bound (as we assume here) or a sub-Gaussian assumption. Without such a tail bound one can adversarially construct a distribution that violates any claimed $\CVaR$ or $\E$ bound by placing a finite probability mass arbitrarily far to the right in the distribution, pulling both $\CVaR$ and $\E$ far enough right to violate the claimed bound. By contrast $\VaR$ ignores the tail, so finite sample bounds on $\VaR$ can be constructed without a priori bounds on the tail. 

Since $J$ is computed as the sum of $T$ stage costs it suffices to simply find bounds on the stage cost and multiply these by $T$ to bound the trajectory cost. For some cost functions bounds may be computed analytically, especially in the case of bounded state and control spaces. Otherwise, in practice one can always clip the value of the cost function between some user-defined bounds, bounding the support of the total cost by construction. Computation of support bounds can be done offline and tight support bounds are not needed, but tighter support bounds on $J$ will lead to tighter bounds on $\mathbb{E}[J]$ and $\textup{CVaR}_\tau(J)$.

Finally, we define a constant that arises in both $\CVaR$ and $\E$ bounds, originating from the application of the DKW bound~\cite{dvoretzky1956asymptotic, massart1990tight} in their derivation, as discussed in the proofs in the Appendix.

\begin{definition}[DKW Gap]
    \label{Def:DKWGap}
    We define the DKW gap as
    \[
        \epsilon(\delta, n) = \sqrt{\frac{-\ln{\delta}}{2n}}.
    \]
\end{definition}

\begin{theorem}[Expected Value Bound] \label{thm: exp_bound}
    Consider  $\delta \in (0,0.5]$, an almost sure upper bound $J_\ub$, and $n$ IID cost samples $J_{1:n}$. Let $k$ be the smallest index such that $\frac{k}{n} - \epsilon \ge 0$. We have the following probabilistic upper bound on $\E[J]$,
    \begin{align}
        \overline{\E} & \coloneqq \epsilon J_{ub} +
        \left(\frac{k}{n}- \epsilon\right)J_{(k)} + \frac{1}{n}\sum_{i = k+1}^nJ_{(i)},
    \end{align}
    which has the property
    \[
        \Pr[\E[J] \le \overline{\E}] \ge 1-\delta.
    \]
    We require the number of samples $n \ge -\frac{1}{2}\ln(\delta)$, to ensure $\epsilon \le 1$. If $k = n$, the summation on the right is ignored. For $n < -\frac{1}{2}\ln(\delta)$ we default to the almost sure bound $J_{\ub}$. 
\end{theorem}

% \begin{proof}
%     The proof for this is in the \textit{Proofs} section of the Appendix.
% \end{proof}

\begin{theorem}[CVaR Bound] \label{thm: cvar_bound}
    Consider $\tau \in [0,1)$, $\delta \in (0,0.5]$, an upper bound $J_\ub$, and $n$ IID cost samples $J_{1:n}$. Let $k$ be the smallest index such that $\frac{k}{n} - \epsilon - \tau \ge 0$. We have the following probabilistic upper bound on $\CVaR_{\tau}$,
    \begin{align}
        \overline{\CVaR}_\tau \coloneqq \frac{1}{1-\tau}\Big[\epsilon J_{ub} + 
        \left(\frac{k}{n}- \epsilon - \tau\right)J_{(k)} + \frac{1}{n}\sum_{i = k+1}^nJ_{(i)}\Big],
    \end{align}
    which has the property
    \[
        \Pr[\CVaR_{\tau}(J) \le \overline{\CVaR}_{\tau}] \ge 1-\delta.
    \]
    We require the number of samples $n \ge -\frac{1}{2}\ln(\delta)/(1-\tau)^2$, to ensure $\epsilon \le 1 - \tau$. If $k = n$, the sum on the right is ignored.  For $n < -\frac{1}{2}\ln(\delta)/(1-\tau)^2$ we default to the almost sure bound $J_{\ub}$.
\end{theorem}

% \begin{proof}
%     The proof for this as well as for lower bounds and two-sided bounds are given in the Appendix.
% \end{proof}

Both the $\CVaR$ and $\E$ bounds above take the form of a sample average over the order statistics, excluding some proportion of the smaller order statistics defined by the DKW gap $\epsilon(\delta, n)$, and including the upper bound $J_\ub$ and the smallest effective order statistic $J_{(k)}$ with special weightings, also determined by $\epsilon(\delta, n)$. Therefore, both of these bounds can be understood as variations on the sample average that typically serves as a proxy for expected value. The great advantage of these bounds is that, unlike a sample average, they rigorously upper bound the unknown quantity ($\CVaR$ or $\E$) with a user-defined probability $1-\delta$ without knowing the underlying probability distribution of $J$.  

Notice that setting $\tau = 0$ in the $\CVaR$ bound gives the $\E$ bound, which is appealing given that this is also true of $\CVaR$ and $\E$ from their definitions in Def.~\ref{Def:CVaR} and Def.~\ref{Def:E} above. We note that the $\CVaR$ bound in Theorem~\ref{thm: cvar_bound} is mathematically equivalent to the one derived in~\cite{cvar_thomas}, but expressed in a different form and derived by different means.

\textcolor{black}{While the bounds described can be performed for any choice of $\delta$ and $\tau$, for a given number of samples $n$, as $\delta$ decreases, the actual bound values will increase as we require the bounds to hold with higher confidence. As $\tau$ increases, the actual bound values will also increase as we seek to bound a larger measure of the cost distribution (noting that $\VaR_{\tau}(J)$ and $\CVaR_{\tau}(J)$ are increasing with respect to $\tau$). Similarly, as $\delta$ decreases and/or as $\tau$ increases, the minimum number of samples needed to yield (non-vacuous) bounds grows.}

Finally, we introduce an upper bound on the probability of failure in a binary cost setting.
\begin{theorem}[Failure Probability Bound] \label{thm: feasibility_bound}
 Given $\delta \in (0,1)$ and $n$ IID Bernoulli samples $J_{1:n}$ (where $J = 1$ denotes failure) with \textcolor{black}{$k = \sum_{i=1}^n J_i$ failures}, we have the following probabilistic upper bound on the probability of failure, $q \coloneqq \Pr[J = 1]$,
    \begin{gather}
        \textcolor{black}{\overline{q} = \max \{q' \in [0,1] \mid \textup{Bin}(k;n,q') \ge \delta \},}
        % \overline{q} = 1 - \min \{p' \in [0,1] \mid \textup{Bin}(k-1;n,p') \le 1 - \delta\},  
    \end{gather}
    which has the property
    \[
        \Pr\left[q \le \overline{q}\right] \geq 1 - \delta.
    \]
\end{theorem}

% \begin{proof}
% See the Appendix.
% \end{proof}

\subsection{Constraint Satisfaction Tests}
The above theorems bound performance measures on the random cost $J$ attained by a robotic system under a given control policy. Now we adapt the above bounds to give a test for constraint satisfaction.
Consider any of the performance measures above applied to the constraint function, $\P(g)$, embedded in a constraint
\[
 \P(g) \le C.
\]
Using the associated finite-sample bound, which we write generically as $\overline{P}$, we define the associated constraint test as $\overline{P} \le C$. We have the following result.
\begin{theorem}[Constraint Test]
\label{Thm:ConstraintTest}
 The test for constraint satisfaction $\overline{P} \le C$ has a false acceptance rate of no more than $\delta$, that is,
 \[
 \Pr[\overline{P} \le C \mid \P(g) > C] \le \delta.
 \]
\end{theorem}

% \begin{proof}
% See the Appendix.
% \end{proof}

The ability to provide a false acceptance guarantee for the constraint tests is a powerful yet natural consequence of the bound error rates. Indeed, we could not readily guarantee a false acceptance rate if using Monte Carlo estimates for the performance measures.

We noted in Section \ref{sec:Problem} that chance constraints:
    \[
    \Pr[g \leq 0] \geq \tau
    \]
can be modeled using either a binary function for $g \leq 0$ or by constraining $\VaR_{\tau}(g)$. Using this equivalence, we can test whether a chance constraint holds via two equivalent methods: 
\begin{itemize}
    \item By forming $\overline{\VaR}_{\tau} = g_{(k)}$ using Theorem \ref{thm: var_bound} and checking whether $g_{(k)} \leq 0$,
    \item By first computing the number of successes/failures to satisfy $g(\X, \U) \leq 0$ to form $\overline{q}$ using Theorem \ref{thm: feasibility_bound} and checking whether $\overline{q} \leq 1 - \tau$.
\end{itemize} 
Both methods also require the same number of minimum samples $n$ to ever accept: $\Bin(n-1;n,\tau) \geq 1-\delta$, although this condition is enforced directly in Theorem \ref{thm: var_bound}.

\textcolor{black}{It may be the case that one wishes to assess a policy's safety, via a constraint satisfaction test $\P(g) \leq C$, and if safe, obtain a bound on the policy's cost performance $\P(J)$. To disentangle these assessments, we would first repeatedly rollout the policy recording $g_i$ to assess $\P(g) \leq C$. If the policy passes the safety test i.e., $\overline{P} \leq C$, we then separately rollout the policy recording $J_i$ to bound $\P(J)$.}

\section{Bound Evaluation Experiments} \label{sec: plan_eval_exp}
\subsection{Performance Bounds}
In Fig.~\ref{fig:valid_stat} we empirically validate the bounds given by Theorems~\ref{thm: var_bound}, \ref{thm: exp_bound}, \ref{thm: cvar_bound}, \ref{thm: feasibility_bound} (subfigures \ref{fig: var_bound}, \ref{fig: exp_bound}, \ref{fig: cvar_bound}, \ref{fig: pr_bound} respectively). We use the error rate $\delta = 0.2$ for all bounds, and the quantile $\tau = 0.7$ for $\VaR_\tau$ and $\CVaR_\tau$. For a fixed policy $\mathcal{U}$, each plot shows the resulting distribution of total cost \textcolor{black}{as a blue histogram generated using $10,000$ simulation rollouts.} \textcolor{black}{Since the bounds are sample-based, they themselves follow a distribution, shown as the overlaid gray histogram. To compute this bound distribution, we repeatedly ($1000$ times) generate the sample-based bound using a fresh batch of $n = 100$ sampled policy rollouts.} The blue dashed vertical line shows the true performance measure \textcolor{black}{we seek to bound},\footnote{Since we do not have access to the true underlying performance measure, we approximate the true measure with a Monte Carlo estimate from the 10,000 simulation rollouts. This is only for visualization purposes.} while the gray dashed vertical line shows the \textcolor{black}{$\delta$ ($0.2$ in this case) quantile of the bound distribution. \footnote{This theoretical quantile is also approximated for visualization as the empirical quantile of the repeated bound generations.} Since the bounds are sample-based, an individual generated bound may be invalid and fall below the true performance measure. We observe this in subfigures~\ref{fig: var_bound} and~\ref{fig: pr_bound} where portions of the gray bound histogram lie left of the blue dashed line. Yet, the theorems guarantee that the bounds hold with probability at least $1 - \delta$ so at least $1 - \delta$ of the bound distribution should exceed the true performance measure. Equivalently, the bound distribution's $\delta$ quantile should exceed the true performance measure. Visually, this means that the gray dashed line should lie right of the blue dashed line. We indeed observe this result in all the subfigures, empirically demonstrating that the sample-based bounds hold with probability at least $1 - \delta$ as guaranteed by the theorems.
}

\begin{figure}
     \centering
     % \subfloat[]{\includesvg[width=\linewidth]{figs/half_cheetah_var.svg}\label{fig: var_bound}}
     \subfloat[]{\includegraphics[width=\linewidth]{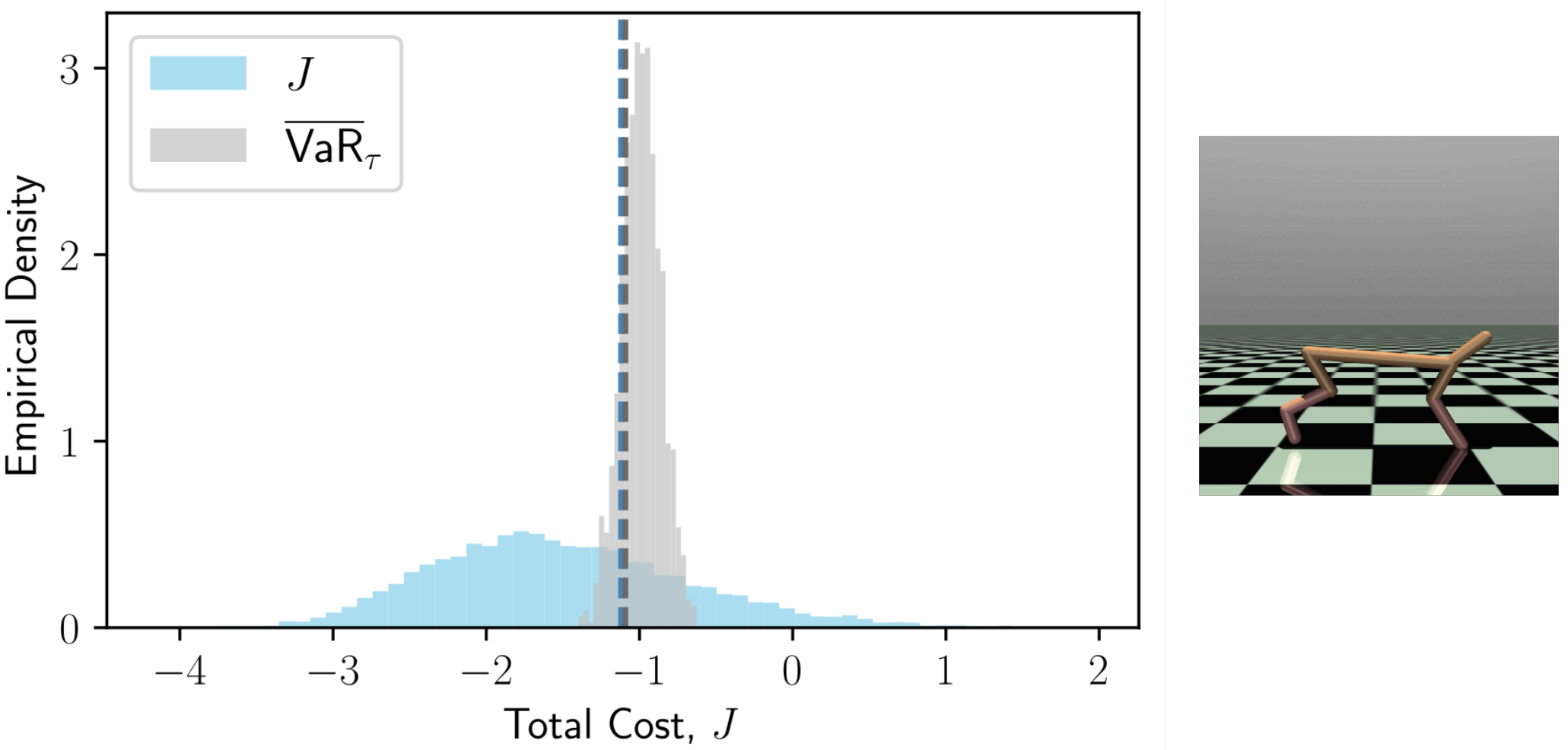}\label{fig: var_bound}}
     \hfill
     % \subfloat[]{\includesvg[width=\linewidth]{figs/ant_exp.svg}\label{fig: exp_bound}}
     \subfloat[]{\includegraphics[width=\linewidth]{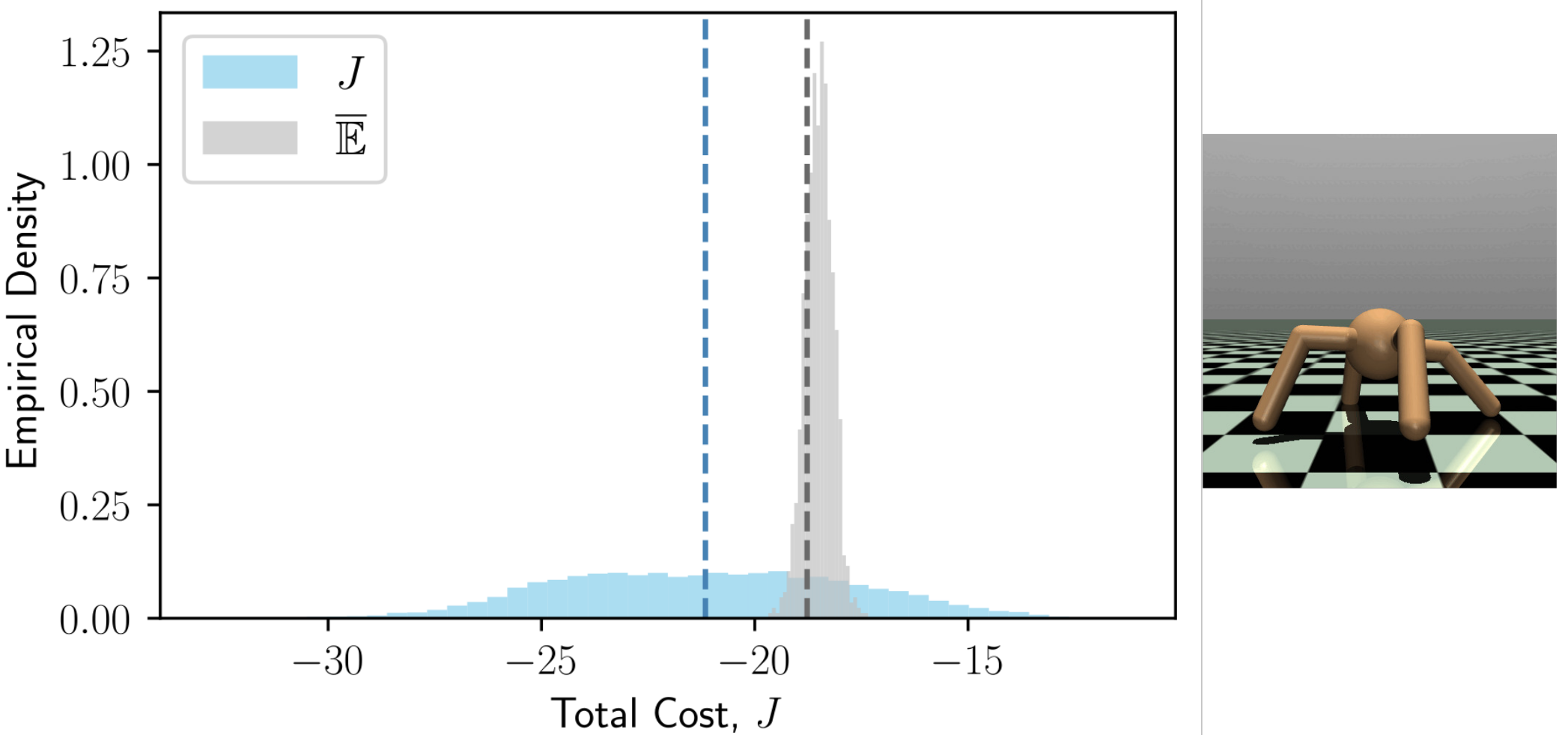}\label{fig: exp_bound}}
     \hfill
     % \subfloat[]{\includesvg[width=\linewidth]{figs/swimmer_cvar.svg}\label{fig: cvar_bound}}
     \subfloat[]{\includegraphics[width=\linewidth]{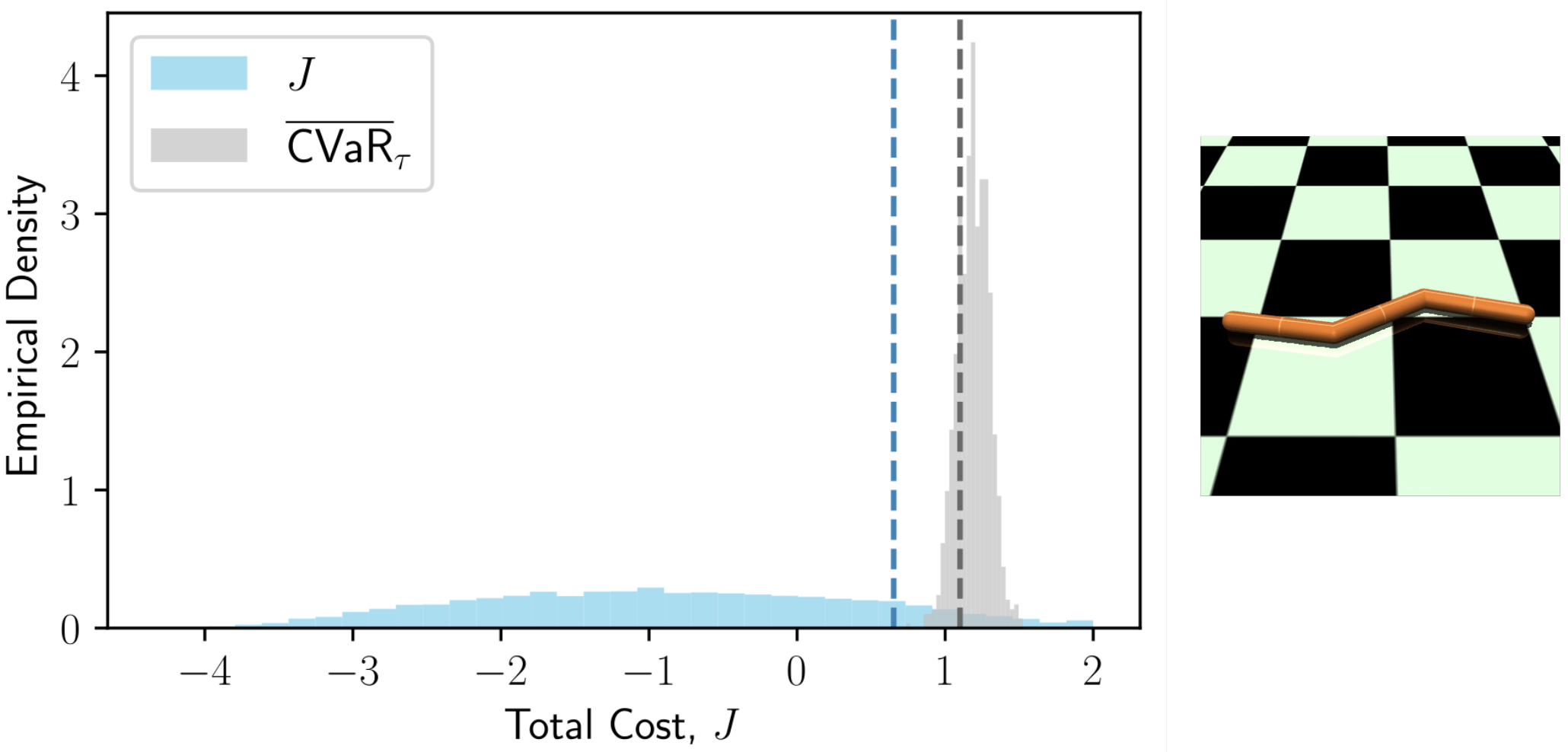}\label{fig: cvar_bound}}
     \hfill
     % \subfloat[]{\includesvg[width=\linewidth]{figs/ant_pr.svg}\label{fig: pr_bound}}
     \subfloat[]{\includegraphics[width=\linewidth]{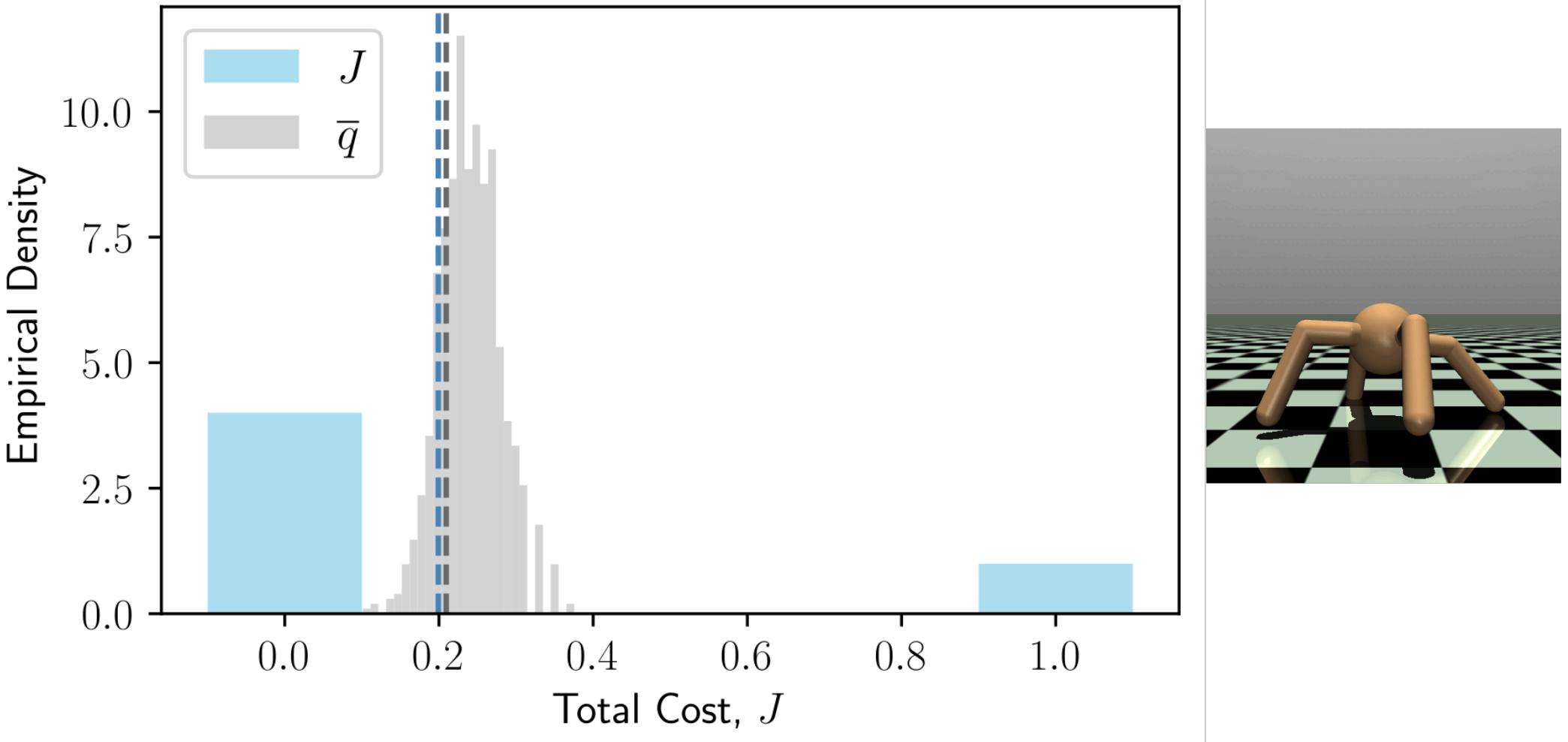}\label{fig: pr_bound}}
    \caption{Empirical validation of the bounds for $\VaR_\tau$, $\E$, $\CVaR_\tau$, and $q$ (subfigures \ref{fig: var_bound}, \ref{fig: exp_bound}, \ref{fig: cvar_bound}, \ref{fig: pr_bound}, respectively). Each plot shows, for a single policy, the empirical distribution of total cost $J$ (blue), along with the distribution of the bound (gray). The blue vertical line shows the true measure we seek to bound and the gray vertical line shows the $\delta$ quantile of the bound distribution. \textcolor{black}{Since our theoretical results ensure the bounds holds with probability $\ge1 - \delta$, the $\delta$ quantile of the bound distribution should exceed the true measure. Thus visually, our} results ensure the gray line is to the right of the blue line, as validated in each plot. \textcolor{black}{The cost} histogram was generated using $10,000$ simulations \textcolor{black}{and the bound histogram was generated by repeatedly computing the bound $1000$ separate times.} In each case $n=100$, $\delta = 0.2$, and $\tau = 0.7$. To demonstrate that the bounds are agnostic to the dynamics, we used the Half Cheetah (\ref{fig: var_bound}), Ant (\ref{fig: exp_bound}, \ref{fig: pr_bound}), and Swimmer (\ref{fig: cvar_bound}) MuJoCo environments.}
    \label{fig:valid_stat}
\end{figure}

To emphasize that the bounds are agnostic to the form of the dynamics, we use a variety of MuJoCo environments for testing the validity of Theorems \ref{thm: var_bound}, \ref{thm: exp_bound}, \ref{thm: cvar_bound}, \ref{thm: feasibility_bound} (Half Cheetah~\cite{half_cheetah}, Ant~\cite{ant}, Swimmer~\cite{swimmer}, and Ant again respectively). Due to limitations of the MuJoCo simulator, in these experiments, the dynamics of the robot are deterministic. Uncertainty comes from the starting state of the robot, which is randomly initialized using a Gaussian distribution centered about a nominal state. This form of uncertainty could realistically arise when a state estimation algorithm (e.g., Kalman filter) is used to provide a distribution over the starting state. The cost functions used are the negative values of the default rewards in MuJoCo, which encourage forward motion and minimal control input. We use clipping of the stage cost to ensure bounded support when computing the expectation and CVaR bounds. For the sparse cost case in subfigure \ref{fig: pr_bound}, we declared success, setting $J = 0$, when the Ant torso remained within the standard height range considered by default in MuJoCo, but still used the continuous cost when optimizing. The open-loop policies considered are obtained as the result of optimizing with the cross-entropy method (CEM)~\cite{mannor2003cem}. Thus, the bounds can be viewed as a \textcolor{black}{probabilistic} guarantee on the optimizer's solution performance under randomized initial conditions. Further experiment details are in the Appendix.

While the expectation and CVaR bound (Theorems~\ref{thm: exp_bound}, \ref{thm: cvar_bound}) are somewhat loose when \textcolor{black}{formed} using the relatively small number of 100 samples, the VaR and probability of failure bounds (Theorems~\ref{thm: var_bound}, \ref{thm: feasibility_bound}) are quite tight. In our later experiments, we focus on showing results using VaR \textcolor{black}{and failure probability} as the relevant statistics. 

\textcolor{black}{In the \textit{Bound Comparison} section of the Appendix we compare the bound distributions (gray) we obtain in Figure~\ref{fig:valid_stat} with the bound distributions one would obtain using different bounds from the literature (specifically those from~\cite{caltech_policy_synth,risk_verification,cvar_kolla}). We show that the bounds we use are less conservative than others (better estimating the unknown performance measure). Using less conservative bounds allows more accurate understanding of policy performance, so practitioners can better decide whether to deploy a policy or devote more resources towards policy synthesis/improvement.}

\subsection{Constraint Satisfaction Tests}
In Fig.~\ref{fig:valid_chance}, using our approach from Theorem \ref{Thm:ConstraintTest}, we show the relationship between probability of our constraint test holding and the probability that the underlying constraint is actually satisfied. To convey the effect of sample size on the test we show theoretical curves for $n \in \{10, 50, 100, 500\}$. We empirically validate the theory for the $n = 10$ case using simulations of the MuJoCo Ant environment. Specifically, the constraint function $g = 0$ is a success if the height of the ant torso remains in the interval [0.5, 1] for an entire rollout (based on the healthy condition specified in MuJoCo), and a failure otherwise. We impose a binary failure probability constraint (\ref{eq: g_p}) requiring that this condition is satisfied with probability at least $\tau = 0.7$ (shown by the gray dotted vertical line in Fig.~\ref{fig:valid_chance}). In other words, we have reformulated a chance constraint as a constraint on the failure probability of a binary $g$. We seek to verify whether or not the provided control policy satisfies this constraint by using the test derived from Theorem~\ref{thm: feasibility_bound}, checking whether the upper bound on the failure probability is sufficiently low $\overline{q} \leq 1 - \tau$. By constructing $\overline{q}$ with user-specified error rate of $\delta = 0.2$ we are guaranteed to have a false acceptance rate no greater than $\delta = 0.2$ (shown by the dashed gray horizonal line in Fig.~\ref{fig:valid_chance}) by Theorem \ref{Thm:ConstraintTest}. To validate the test over a range of satisfying and violating policies, we use 20 open-loop policies generated by randomly sampling control actions. 

For each policy, we repeatedly (1000 times) collect a fresh set of $n = 10$ samples of the constraint function $g_{1:n}$ by executing the control actions from random initial conditions in the Ant environment, and apply Theorem \ref{Thm:ConstraintTest} with cutoff $C = 1 - \tau$ and bound $\overline{P} = \overline{q}$ computed from Theorem~\ref{thm: feasibility_bound}. We call each such bound computation a trial. We use the empirical fraction of trials for which we concluded the chance constraint holds (based on the test) to approximate the unknown true probability. This fraction provides the y-coordinate for the associated point in the figure. We obtain the associated x-coordinate, the ``ground truth" constraint satisfaction probability $\Pr[g = 0]$, through 1000 Monte Carlo simulations. Specifically, we simulated the policy 1000 times and recorded the empirical fraction of the resulting trajectories that were a success, obtaining $g = 0$.

The tight agreement between the theoretical and empirical results show that we can use our method to provide a sampling-based certification method for policy constraint satisfaction. In particular, we observe that both in the empirical and theoretical curves whenever the constraint fails to hold, the acceptance probability is below $\delta$: to the left of the vertical line at $\tau = 0.7$ all curves lie below the horizontal line at $\delta = 0.2$, avoiding the region shaded in red. 

Similar figures could have been generated for constraint tests on other performance measures with continuous $g$ e.g., a CVaR constraint. However, we would not easily be able to generate a corresponding theoretical curve. With Theorem~\ref{thm: feasibility_bound}, we can compute the theoretical curves using the Binomial distribution and the true probability of success.

\begin{figure}
    \centering
    % \includesvg[width=\linewidth]{figs/valid_chance.svg}
    \includegraphics[width=\linewidth]{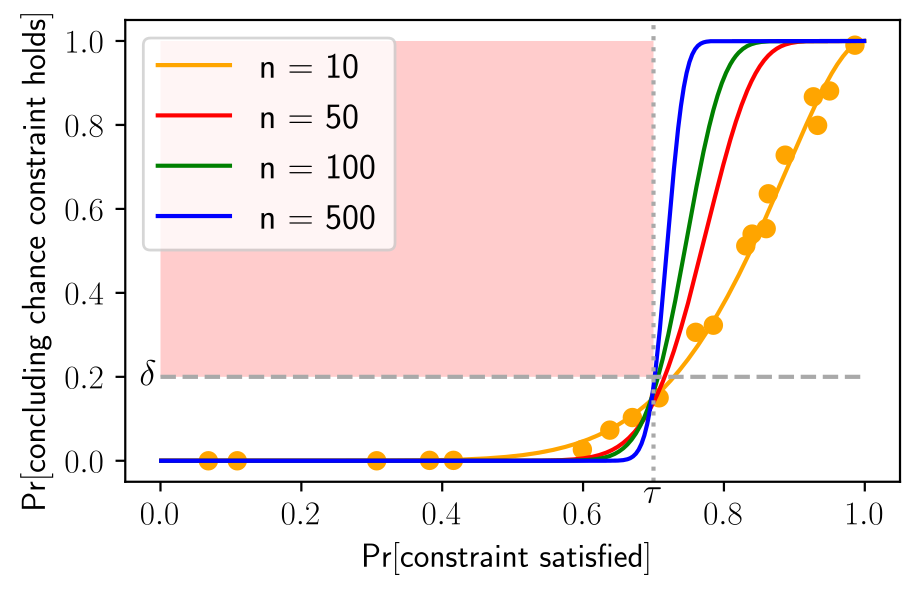}
    \caption{Visualization of Theorem \ref{Thm:ConstraintTest}, and empirical validation of the theorem, applied to testing whether a chance constraint holds. Each curve represents the probability of accepting that the chance constraint holds (y-axis) given the true probability of the underlying trajectory constraint being satisfied (x-axis). The validity of the theorem is demonstrated by each curve being below $\delta$ when the chance constraint fails to hold i.e., $\Pr[\textup{constraint satisfied}]$ is below $\tau$. Visually, the false acceptance is guaranteed to be below $\delta$ so that the curves avoid the region shaded in red in the figure. Here we use $\delta=0.2$ (horizontal line), $\tau=0.7$ (vertical line). Furthermore, as the sample size ($n$) increases, the curve approaches a step function, i.e. we obtain a perfect discriminator. In addition, for $n=10$ we plot empirical results from the Ant environment where the vertical position of the Ant torso always being between $[0.5, 1]$ with probability $0.7$ is the chance constraint we seek to assess. The $x$ and $y$ coordinates for each orange dot are separately estimated using an average taken over 1000 simulation runs.}
    \label{fig:valid_chance}
\end{figure}

%%%%%%%%%%%%%%%%%%%%%%%%%%%%%%%%%%%%%%%%

%%%%% Sensitivity Analysis %%%%%%%%%%
\section{\textcolor{black}{Bound Sensitivity to Distribution Shifts}}
\label{sec: bound_sensitivity}
\textcolor{black}{
In this section we give analytical expressions for the effect that changes in cost distributions have on the confidence level of the bounds presented in Section \ref{sec: plan_evaluation}. The setting we consider is when the  distribution of cost based on our simulator does not match the true distribution of cost when the policy is deployed in the real world. Using samples from the simulator cost distribution, we construct performance bounds with confidence level $1-\delta_{sim}$. In the following subsections we give expressions for how this confidence level changes (to $1 - \delta_{true}$) when the distribution of cost in the real world doesn't match that of the simulator. We present the sensitivity relationships in this section as corollaries of the bound theorems in Section \ref{sec: plan_evaluation}. All proofs are deferred to the Appendix.
}

\textcolor{black}{
It is important to note the inherent trade-off between efficiency of a bound and its robustness to distribution shift. If a bound very precisely estimates the unknown parameter, it will be more sensitive to distribution shifts. If one anticipates a certain level of distribution shift, it is not appropriate to use less precise bounds, instead, in Section~\ref{sec: robust_bounds} we detail how to modify the bounds presented in this paper to be appropriately robust to distribution shift. 
}

\textcolor{black}{In this paper we measure distribution shift between the simulated cost distribution $\mathcal{D}_{sim}$ and the true cost distribution $\mathcal{D}_{true}$ according to the one-sided Kolmogorov-Smirnov (KS) distance for distributions~\cite{lehmann_textbook},
\begin{align}
    \sup_x \ \text{CDF}_{\mathcal{D}_{sim}}(x) - \text{CDF}_{\mathcal{D}_{true}}(x).
\end{align}
We consider the one-sided KS distance because it captures when a cost CDF shifts downwards, relating to a harmful distribution shift of higher cost more often.}

\begin{corollary}[\textcolor{black}{Sensitivity of VaR Bounds}] \label{cor: sens_var}
    \textcolor{black}{
Suppose we construct $\overline{\textup{VaR}}_\tau$ with samples from the simulated cost distribution, $J_{sim} \sim \mathcal{D}_{sim}$. Then suppose that that the true cost distribution, $\mathcal{D}_{true}$, is close to $\mathcal{D}_{sim}$ in the one-sided KS distance,
\begin{align}
    \sup_x \ \CDF_{\mathcal{D}_{sim}}(x) - \CDF_{\mathcal{D}_{true}}(x) \le \alpha.
\end{align}
Then we have,
\begin{subequations} \label{eq:var_sensitivty}
    \begin{gather}
    \Pr[\textup{VaR}_{\tau}(J_{true}) \le \overline{\textup{VaR}}_\tau] \ge 1-\delta_{true}, \\
    \delta_{true} = 1 - \textup{Bin}(k^*-1;n,\tau+\alpha), \label{eq:tau'_bound}\\
    k^* = \min \{k \in \{1,\ldots,n\} \mid  \Bin(k-1;n,\tau) \ge 1 - \delta_{sim}\}.
    \end{gather}
\end{subequations}
}
\end{corollary}

\textcolor{black}{In Figure \ref{fig:var_sensitivity} we plot the sensitivity of the VaR bound to distribution shifts imposed by misspecifying a parameter for the Half Cheetah environment, the noise parameter $\sigma$ over the initial state distribution. The empirical confidence level of the bound $\Pr[\VaR_{\tau}(J_{true}) \leq \overline{\VaR}_{\tau}]$ always exceeds the theoretically predicted $1 - \delta_{true}$ using Corollary \ref{cor: sens_var}. However, the theoretical prediction is pessimistic when $\sigma$ is decreased from the nominal value $\sigma_{sim}$ as this distribution shift actually results in increased confidence level. A tighter theoretical bound may be obtained if we assume knowledge of the precise $\tau'$ such that $\VaR_{\tau}(J_{true}) = \VaR_{\tau'}(J_{sim})$ and then use $\tau'$ instead of the larger $\tau + \alpha$ in Eq.~\ref{eq:tau'_bound}}.

\begin{figure}
    \centering
    % \includesvg[width=\linewidth]{figs/shift_fixed_var.svg}
    \includegraphics[width=\linewidth]{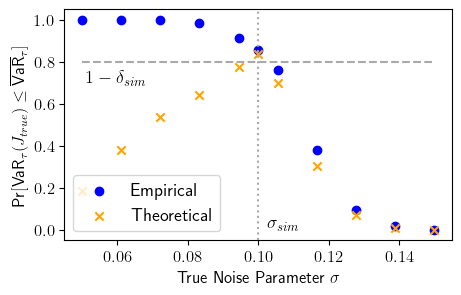}
    \caption{\textcolor{black}{Confidence level of the VaR bound as a sim-to-real mismatch is varied. The parameter $\sigma$ controls the standard deviation of the initial state distribution in the Half Cheetah environment. When $\sigma > \sigma_{sim}$, the true confidence level of the bound degrades. When $\sigma < \sigma_{sim}$, the true confidence level of the bound strengthens. In blue we plot the empirical confidence levels estimated by varying $\sigma$, using $10,000$ simulations to estimate $\textup{VaR}_{\tau}(J_{true})$, and using $1000$ realizations of $\overline{\textup{VaR}}_{\tau}$. In orange we plot the minimum confidence level guaranteed by Equation \ref{eq:var_sensitivty}. The theoretical sensitivity guarantee is valid, always lower than the empirical confidence, but is pessimistic when $\sigma < \sigma_{sim}$ as even though the one-sided KS distance $\alpha > 0$ the distribution shift actually results in a higher confidence level.}}
    \label{fig:var_sensitivity}
\end{figure}

% We repeatedly compute the bound $\overline{\textup{VaR}}_{\tau}$ $1000$ times, each time using $n = 100$ policy rollouts using $\sigma_{sim}$. For each plotted point we run $10,000$ simulations with parameter $\sigma$ to estimate the true cost distribution and associated $\textup{VaR}_{\tau}(J_{true})$. From this, for each $\sigma$, we compute the empirical confidence level of the bound as the fraction of generated bounds using $\sigma_{sim}$ satisfying $\textup{VaR}_{\tau}(J_{true}) \leq \overline{\textup{VaR}}_{\tau}$. Using the empirical cost distributions for $\sigma_{sim}$ and $\sigma$, we estimate the KS distance $\alpha$ to obtain the theoretical confidence level using Equation \ref{eq:var_sensitivty}. The theoretical sensitivity guarantee is valid, always lower than the empirical confidence, but is pessimistic when $\sigma < \sigma_{sim}$ as even though the one-sided KS distance $\alpha > 0$ the distribution shift actually results in a higher confidence level.

\begin{corollary}[\textcolor{black}{Sensitivity of $\mathbb{E}$ and CVaR Bounds}] \label{cor: sens_cvar}
    \textcolor{black}{
Suppose we construct $\overline{\textup{CVaR}}_\tau$ with samples from the simulated cost distribution, $J_{sim} \sim \mathcal{D}_{sim}$. Suppose that the true cost distribution, $\mathcal{D}_{true}$, is close to $\mathcal{D}_{sim}$ in the one-sided KS distance,
\begin{align}
    \sup_x \ \CDF_{\mathcal{D}_{sim}}(x) - \CDF_{\mathcal{D}_{true}}(x) \le \alpha \le \sqrt{\frac{-\ln(2\delta_{sim})}{2n} }.
\end{align}
Then we have,
\begin{subequations}
    \begin{gather}
    \Pr[\textup{CVaR}_\tau(J_{true}) \le \overline{\textup{CVaR}}_\tau] \ge 1-\delta_{true}, \\
    \delta_{true} = e^{-2n(\epsilon - \alpha)^2}, \\
    \epsilon = \sqrt{\frac{-\ln \delta_{sim}}{2n}}.
\end{gather} \label{eq: cvar_sensitivity}
\end{subequations}
}
\end{corollary}
\textcolor{black}{The condition that $\alpha \le \sqrt{\frac{-\ln(2\delta_{sim})}{2n} }$ is an artifact of the DKW bound holding for $\delta \in (0,0.5]$. One can remove this condition (at the expense of closed-form expressions for the bounds) by using the Kolmogorov-Smirnov approach~\cite{birnbaum1951} in place of the DKW approach when constructing the bounds. Lastly, since we treat expected value as a special case of CVaR, the relationships in \ref{eq: cvar_sensitivity} also hold for the expected value bound.}

\begin{corollary}[\textcolor{black}{Sensitivity of Failure Probability Bounds}] \label{cor: sens_fail_prob}
    \textcolor{black}{
Suppose we construct $\overline{q}$ by observing $k$ failures out of $n$ samples from a Bernoulli distribution with probability of failure $q_{sim}$. Then suppose that that the true cost distribution, $\mathcal{D}_{true}$, is close to $\mathcal{D}_{sim}$ in the one-sided KS distance,
\begin{align}
    \sup_x \ \CDF_{\mathcal{D}_{sim}}(x) - \CDF_{\mathcal{D}_{true}}(x) \le \alpha,
\end{align}
i.e. $q_{true} - q_{sim} \le \alpha$. Then we have,
\begin{subequations}
    \begin{gather}
    \Pr[q_{true} \le \overline{q}] \ge 1-\delta_{true}, \\
    \delta_{true} = \textup{Bin}(k_{\alpha}^*-1;n,q_{sim}), \\
    k^*_{\alpha} = \min \{k \in \{0,\ldots,n\} \mid \textup{Bin}(k;n,q_{sim}+\alpha) \ge \delta_{sim} \}.
\end{gather}
\end{subequations}
}
\end{corollary}

%%%%%%%%%%%%%%%%%%%%%%%%%%%%%%%%%%%%%%%%

%%%%% Robustness Analysis %%%%%%%%%%
\section{\textcolor{black}{Constructing Robust Bounds}} \label{sec: robust_bounds}

\textcolor{black}{
Building on the bound sensitivity results, in this section we show how to construct robust bounds which will hold with the desired confidence level ($1-\delta$) even when there is mismatch between the simulated and real cost distributions. To construct robust bounds, we anticipate the potential cost distribution shift and appropriately increase the bounds, slightly modifying those described in Section \ref{sec: plan_evaluation}. 
}

\textcolor{black}{
Specifically, we again assume that the the true and simulated cost distributions are within $\alpha$ in terms of the one-sided Kolmogorov-Smirnov distance,
\begin{gather}
    \sup_x \ \text{CDF}_{\mathcal{D}_{sim}}(x) - \text{CDF}_{\mathcal{D}_{true}}(x) \le \alpha
\end{gather}
where $\mathcal{D}_{sim}, \mathcal{D}_{true}$ are the simulated and true cost distributions. Although precisely specifying $\alpha$ is problem-dependent and may be challenging, these robust bounds provide a principled mechanism for acting conservatively when the simulator is known to be inaccurate. Each of the expressions we give are presented as corollaries of the bound theorems in Section~\ref{sec: bound_sensitivity} and all proofs are deferred to the Appendix.
}

% \subsection{\textcolor{black}{Robust VaR Bounds}}
% \textcolor{black}{
% To construct a robust sample-based bound on the true $\VaR_{\tau}(J_{true})$, using samples from the simulator, we simply inflate to form a bound on the $\tau' = \tau + \alpha$ quantile of the simulated cost $\VaR_{\tau'}(J_{sim})$. Using $\tau'$ will guarantee that the bound holds with at least probability $1-\delta$ even when the true cost distribution differs from the simulated cost distribution by $\alpha$. We formally state this result in the Theorem below. While its complete proof is provided in the Appendix, the method used is to show that $F_{true}(x) \geq F_{sim}(x) - \alpha \ \forall x$ implies $\VaR_{\tau}(J_{true}) \leq \VaR_{\tau'}(J_{sim})$ for $\tau' = \tau + \alpha$ and then apply the previous VaR sensitivity result.
% }

% In contrast with using bound sensitivity results to directly inflate the required simulation error rate $\delta_{sim}$, instead inflating the $\tau$ level can be favorable as this shifts the entire bound distribution.

\begin{corollary}[\textcolor{black}{Robust VaR Bounds}] \label{cor: robust_var}
    \textcolor{black}{Consider $\tau, \delta \in (0,1)$ and $n$ IID cost samples from the simulator $J_{1:n} \sim \mathcal{D}_{sim}$ and assume that $\sup_x\text{CDF}_{\mathcal{D}_{sim}}(x) - \text{CDF}_{\mathcal{D}_{true}}(x) \le \alpha$. Then, we have the $\alpha$-robust VaR bound,
\begin{subequations}
    \begin{gather}
        \Pr[\VaR_\tau(J_{true}) \le \overline{\VaR}_{\tau}(\alpha)] \ge 1-\delta, \\
        \overline{\VaR}_{\tau}(\alpha) = \overline{\VaR}_{\tau + \alpha},
    \end{gather}
\end{subequations}
where $\overline{\VaR}_{\tau + \alpha}$ is constructed using $J_{1:n}$ as in Theorem \ref{thm: var_bound} to hold with probability $1-\delta$.
}
\end{corollary}

% In light of this result, one can think of a bound on CVaR being a robust bound on the expected value, specifically one that is robust to distribution shifts in the Kolmogorov-Smirnov distance. This is analogous to a bound on entropic value at risk being a robust bound on the expected value with respect to distribution shifts in the Kullback–Leibler divergence~\cite{entropic}.

% \textcolor{black}{
% \begin{theorem}[Robust VaR Bound] \label{thm: robust_var_bound}
%     Consider $\tau, \delta \in (0,1)$ and $n$ IID cost samples from the simulator $J_{1:n} \sim F_{sim}$ and assume that $\sup_x\text{CDF}_{\mathcal{D}_{sim}}(x) - \text{CDF}_{\mathcal{D}_{true}}(x) \le \alpha$. Then, with $\tau' = \tau + \alpha$ forming $\overline{\VaR}_{\tau'}$ as in Theorem \ref{thm: var_bound} using the simulated costs $J_{1:n}$ yields a probabilistic upper bound $\overline{\VaR}_{\tau}(\alpha) \coloneqq \overline{\VaR}_{\tau'}$ on the true $\VaR_\tau(J_{true})$ satisfying
%     \[
%     \Pr[\VaR_\tau(J_{true}) \le \overline{\VaR}_{\tau}(\alpha)] \ge 1-\delta.
%     \]
% \end{theorem}
% }

% \textcolor{black}{
% \begin{proof}
%     The proof is given in the Appendix.
% \end{proof}
% }

% \textcolor{black}{
% To address simulation mismatch in the expected value and CVaR case, we modify the original bound by further lowering the DKW CDF bound by $\alpha$ when integrating to form the bound. Thus, the effect is to form a larger bound where the DKW gap $\epsilon(\delta, n)$ is replaced by $\epsilon' = \epsilon + \alpha$. For completeness, we present the modified form in the below Theorem.
% }

\begin{corollary}[\textcolor{black}{Robust $\mathbb{E}$ and CVaR Bounds}] \label{cor: robust_cvar}
    \textcolor{black}{Consider $\tau \in [0,1)$, $\delta \in (0,0.5]$, an upper bound $J_\ub$, and $n$ IID cost samples from the simulator $J_{1:n} \sim \mathcal{D}_{sim}$ and assume that $\sup_x \text{CDF}_{\mathcal{D}_{sim}}(x) - \text{CDF}_{\mathcal{D}_{true}}(x) \le \alpha$. Then, replacing the DKW gap $\epsilon(\delta,n)$ by $\epsilon' = \epsilon(\delta, n) + \alpha$ when applying Theorem \ref{thm: cvar_bound} with $J_{1:n}$ we have the $\alpha$-robust CVaR bound,
\begin{subequations}
    \begin{gather}
        \Pr[\CVaR_{\tau}(J_{true}) \le \overline{\CVaR}_{\tau}(\alpha)] \ge 1-\delta, \\
        \overline{\CVaR}_{\tau}(\alpha) =  \nonumber \\
        \frac{1}{1-\tau}\Big[\epsilon' J_{ub} + 
        \left(\frac{k}{n}- \epsilon' - \tau\right)J_{(k)} + \frac{1}{n}\sum_{i = k+1}^nJ_{(i)}\Big],
    \end{gather}
\end{subequations}
where $k$ is the smallest index such that $\frac{k}{n}- \epsilon' - \tau \geq 0$.}
\end{corollary}

% \textcolor{black}{
% \begin{theorem}[Robust CVaR Bound] \label{thm: robust_cvar_bound}
%     Consider $\tau \in [0,1)$, $\delta \in (0,0.5]$, an upper bound $J_\ub$, and $n$ IID cost samples from the simulator $J_{1:n} \sim F_{sim}$ and assume that $\sup_x \text{CDF}_{\mathcal{D}_{sim}}(x) - \text{CDF}_{\mathcal{D}_{true}}(x) \le \alpha$. Then, replacing the DKW gap $\epsilon(\delta,n)$ by $\epsilon' = \epsilon(\delta, n) + \alpha$ when applying Theorem \ref{thm: cvar_bound} with the simulated costs $J_{1:n}$ we have the following probabilistic upper bound $\overline{\CVaR}_{\tau}(\alpha)$ on the true $\CVaR_{\tau}(J_{true})$,
%     \begin{gather}
%         \overline{\CVaR}_{\tau}(\alpha) \nonumber
%         \coloneqq \\
%         \frac{1}{1-\tau}\Big[\epsilon' J_{ub} + 
%         \left(\frac{k}{n}- \epsilon' - \tau\right)J_{(k)} + \frac{1}{n}\sum_{i = k+1}^nJ_{(i)}\Big],
%     \end{gather}
%     which has the property
%     \[
%         \Pr[\CVaR_{\tau}(J_{true}) \le \overline{\CVaR}_{\tau}(\alpha)] \ge 1-\delta.
%     \]
%     Now, $k$ is taken as the smallest index such that $\frac{k}{n}- \epsilon' - \tau \geq 0$.
% \end{theorem}
% }

% \textcolor{black}{
% \begin{proof}
%     The proof is given in the Appendix.
% \end{proof}
% }

 \textcolor{black}{Taking $\tau = 0$ immediately yields a similar robust bound for the expected value.}

% \textcolor{black}{Lastly, we can similarly modify the failure probability bound in Theorem \ref{thm: feasibility_bound} to account for simulator mismatch. In this special case of sparse cost $J$ is Bernoulli distributed so requiring that $\sup_x \text{CDF}_{\mathcal{D}_{sim}}(x) - \text{CDF}_{\mathcal{D}_{true}}(x) \le \alpha$ is equivalent to $(1 - q_{sim}) - (1 - q_{true}) \leq \alpha$ i.e., $q_{sim} + \alpha \geq q_{true}$. Therefore, we can construct a robust failure probability bound by simply adding $\alpha$ to our previous bound for $q_{sim}$. Instead of doing so after construction, we opt to do so during bound construction which is equivalent except that it results in bounds always in $[0,1]$. The result is summarized in the below Theorem.
% }

\begin{corollary}[\textcolor{black}{Robust Failure Probability Bounds}] \label{cor: robust_fail_prob}
    \textcolor{black}{Consider $\delta \in (0,1)$ and $n$ IID Bernoulli samples $J_{1:n}$ obtained in simulation (where $J = 1$ denotes failure) with $k = \sum_{i=1}^n J_i$ failures and assume $\sup_x \text{CDF}_{\mathcal{D}_{sim}}(x) - \text{CDF}_{\mathcal{D}_{true}}(x) \le \alpha$ (i.e., $q_{true} \le q_{sim} + \alpha$). Then, we have the $\alpha$-robust failure probability bound,
\begin{subequations}
    \begin{gather}
    \Pr\left[q_{true} \le \overline{q}(\alpha)\right] \geq 1 - \delta, \\
        \overline{q}(\alpha) = \max \{q' \in [0,1] \mid \textup{Bin}(k;n,q'-\alpha) \ge \delta \}.
    \end{gather}
\end{subequations}
}
\end{corollary}

\textcolor{black}{In Figure \ref{fig:pr_robust} we plot the confidence level of robust failure probability bounds constructed with different $\alpha$ as we impose a distribution shift by misspecifying the noise parameter $\sigma$ over the initial state distribution for the Ant environment. Whenever the distribution shift is within the allowed tolerance ($q_{true} \leq q_{sim} + \alpha$), as designated by a circle, the confidence level of the bound remains at least $1 - \delta$ confirming Corollary~\ref{cor: robust_fail_prob}.}

\begin{figure}
    \centering
    % \includesvg[width=\linewidth]{figs/robust_fixed_pr.svg}
    \includegraphics[width=\linewidth]{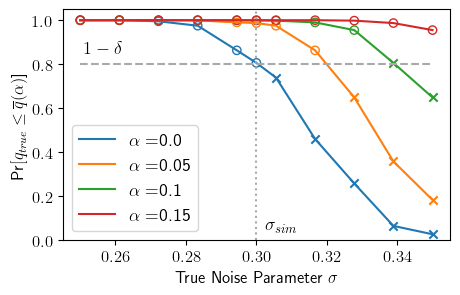}
    \caption{\textcolor{black}{Confidence level of robust failure probability bounds as a sim-to-real mismatch is varied. Similar to Fig.~\ref{fig:var_sensitivity}, we plot the empirical confidence level associated with robust bounds constructed with different tolerances $\alpha \in [0, 0.05, 0.1, 0.15]$ as we vary the parameter $\sigma$ controlling the standard deviation of the initial state distribution in the Ant environment. We observe that even under distribution shift, whenever $q_{true} \leq q_{sim} + \alpha$ (denoted by circles), the associated robust bound $\overline{q}(\alpha)$ holds with confidence level at least $1-\delta$, as expected from Corollary \ref{cor: robust_fail_prob}. The plotted points are generated by varying $\sigma$, using $10,000$ simulations to estimate $q_{true}$, and using $1000$ realizations of $\overline{q}(\alpha)$.}}
    \label{fig:pr_robust}
\end{figure}

\section{Policy Selection} \label{sec: multi-hyp}
Thus far we have presented a method for rigorously assessing the quality of a policy by placing bounds on performance measures of the trajectory cost. A natural extension is to use these bounds when selecting between several candidate policies. In this case, we must take care to apply an appropriate correction to the bounds for the resulting bound \textcolor{black}{on the chosen policy} to remain probabilistically valid. To illustrate the need for a correction, consider a thought experiment where $100$ identical policies are considered as candidates for execution. \textcolor{black}{For each of these identical policies, we generate a fresh set of stochastic simulation rollouts and use these to compute a performance bound. We then choose to execute the policy achieving the lowest bound. Of course, this a false choice since all the policies are the same. However, because the rollouts are stochastic, the computed bounds will fluctuate even though all used the same policy. Choosing the lowest bound among the $100$ will thus give an artificially low performance bound for the associated policy as the chosen bound is the result of a lucky draw of rollouts. Therefore, we cannot expect the resulting bound to still hold with probability at least $1 - \delta$. In other words, since the bounds hold probabilistically, one bound among the $100$ is likely to be overly optimistic by chance, and may even be lower than the true \textcolor{black}{performance measure}, i.e., an invalid bound.} To remedy this problem we explain a statistical correction for comparing bounds among a set of candidate policies.
% \subsection{Coverage Correction}
Consider $m$ (not necessarily independent) policies $\mathcal{U}_{1:m}$ where
% \begin{align}
%     \mathcal{U}_i = (U^{(i)}_0, \ldots, U^{(i)}_{T-1}).
% \end{align}
for each policy $\mathcal{U}_i$ we obtain $n$ IID trajectory rollouts, obtaining trajectory cost samples $J^{(i)}_{1:n}$. Let the performance measure we are interested in be denoted by $\mathcal{P}$ (e.g. $\mathcal{P} = \VaR_\tau$). Then, using the samples $J_{1:n}^{(i)}$ and the results from Section \ref{sec: plan_evaluation}, an individual upper bound $\overline{P}^{(i)}$ is computed for each policy, which has corresponding unknown true performance $P^{(i)}$. Now suppose we plan to execute the policy with least upper bound, that is,
\begin{align}
    \mathcal{U}^* = &\argmin_{\mathcal{U}_i \in \mathcal{U}_{1:m}} \quad \overline{P}^{(i)} .
\end{align}
Take $\overline{P}^*$ as the individual upper bound associated with $\mathcal{U}^*$ and $P^*$ as the true statistic (e.g. true $\VaR_\tau$) associated with $\mathcal{U}^*$. We are interested in understanding with what probability $\overline{P}^*$ upper bounds $P^*$, and formalize this in the below result.

\begin{theorem}[Uncorrected \textcolor{black}{Confidence Level}] \label{thm: uncor_cov}
    Suppose we are given $m$ policies $\mathcal{U}_{1:m}$, with associated unknown true performance $\{P^{(i)}\}_{i=1}^m$ and associated probabilistic bounds $\{\overline{P}^{(i)}\}_{i=1}^m$ individually holding with \textcolor{black}{confidence level} $1-\delta$ i.e., 
    \begin{align}
        \Pr[P^{(i)} \leq \overline{P}^{(i)}] \geq 1-\delta \quad \forall i.
    \end{align} 
    Let $\overline{P}^*$ be the lowest probabilistic bound and let $P^*$ be the associated true statistic i.e., 
    \begin{subequations}
        \begin{gather}
            i^* = \argmin_{i \in \{1,\ldots,m\}} \overline{P}^{(i)}, \\ 
            \overline{P}^* = \overline{P}^{(i^*)}, \\
            P^* = P^{(i^*)}.
        \end{gather}
    \end{subequations}
    Then, $\overline{P}^*$ bounds $P^*$ with probability at least $(1-\delta)^m$ i.e., 
    \begin{align}
        \Pr[P^* \leq \overline{P}^*] \geq (1-\delta)^m.
    \end{align}
\end{theorem}

Before proceeding, we consider two limiting cases of Theorem \ref{thm: uncor_cov}. Temporarily assume that the \textcolor{black}{individual} bounds \textcolor{black}{hold with probability $1 - \delta$} exactly i.e.,
\begin{align}
    \Pr[P^{(i)} \leq \overline{P}^{(i)}] = 1-\delta
\end{align}
then 
\begin{enumerate}
    \item When each bound is bounding the same statistic (i.e., $P^{(1)} = \ldots = P^{(m)}$) then
    \begin{gather}
        \Pr[P^* \le \overline{P}^*] = (1-\delta)^m,
    \end{gather}
    where this follows as all bounds must hold for the minimum bound to hold in this case.
    \item When the cost samples $J_{1:n}^{(i)}$ associated with each policy $\mathcal{U}_i$ are confined to disjoint intervals (i.e., for each $\mathcal{U}_i$, $J \in [J_{lb}^{(i)}, J_{ub}^{(i)}] \coloneqq D_i$ with $D_i \cap D_j = \emptyset \ \forall i \neq j$) then
    \begin{gather}
        \Pr[P^* \le \overline{P}^*] = 1-\delta
    \end{gather}
    where this follows as only one bound, the one generated using the policy $\mathcal{U}_i$ having the lowest interval $D_i$, needs to hold for the minimum bound to hold in this case. 
\end{enumerate}
These two limiting cases show the extremes we must consider when making a correction for multiple policies. In the best case (case 2), \textcolor{black}{the resulting bound $\overline{P}^*$ still holds with probability $1 - \delta$ i.e., the error rate is unchanged}, while in the worst case (case 1), \textcolor{black}{the error rate can increase dramatically} for a large set of candidates. Since we want to avoid distributional assumptions, we must consider the worst case scenario, in which case Theorem \ref{thm: uncor_cov} is tight.

Theorem \ref{thm: uncor_cov} captures the worst case when comparing multiple policies and selecting the policy with lowest bound. Thus, we can use it to correct for the \textcolor{black}{comparison} by first inflating the \textcolor{black}{required probability that each individual bound hold}, as described in the below Theorem.

\begin{theorem}[\textcolor{black}{Multi-Policy Bound} Correction] \label{thm: cor_cov}
    The resulting bound $\overline{P}^*$ obtained when selecting the lowest bound among multiple policies holds with probability at least $1 - \delta$ if each individual bound $\overline{P}^{(i)}$ is inflated to satisfy
    \begin{subequations}
        \begin{align}
        \Pr[P^{(i)} \le \overline{P}^{(i)}] \ge 1-\Bar{\delta}, \\
        \Bar{\delta} = 1 - (1-\delta)^{1/m}. \label{eq: correction}
    \end{align}
    \end{subequations}
\end{theorem}

\begin{remark}[Multi-Hypothesis Connection]
    This correction is identical to the Šidák correction~\cite{abdi2007} for testing multiple hypotheses. Although we are only interested in ensuring that the minimizing bound be probabilistically valid, we end up needing to inflate as if we required all bounds to hold due to case 1 (the typical use case of multi-hypothesis correction).
\end{remark}

\textcolor{black}{
\begin{remark}[Considering Shared Rollout Seeds]
In our approach, for each policy we generate a bound using a fresh set of stochastic simulation rollouts. However, one might consider instead fixing a random seed e.g., fixing a set of sampled environments, which is then reused when generating the rollouts for each policy. Even with this modification, a multi-hypothesis correction is still needed. In fact, since under this fixed seed procedure, the rollouts for each policy would no longer be independent, we would have to resort to the weaker Bonferroni correction~\cite{abdi2007} over the Šidák.
\end{remark}
}

As mentioned previously, introducing the necessity of this bound correction is one of the contributions of this paper. In Fig.~\ref{fig:plan_selection} we summarize the process for selecting the policy with the lowest performance bound while retaining statistical validity using the multi-hypothesis correction. In Section \ref{sec:examples} we show the necessity of the correction by comparing against the naive uncorrected bound in the setting of object manipulation.
% While~\cite{caltech_policy_synth, risk_verification} have similar methods for bounding the performance of control policies, they describe procedures for using these bounds in optimizing for a policy without recognizing the need for a correction. 
 
\begin{figure}
    \centering
    % \includesvg[width=\linewidth]{figs/plan_overview.svg}
    \includegraphics[width=\linewidth]{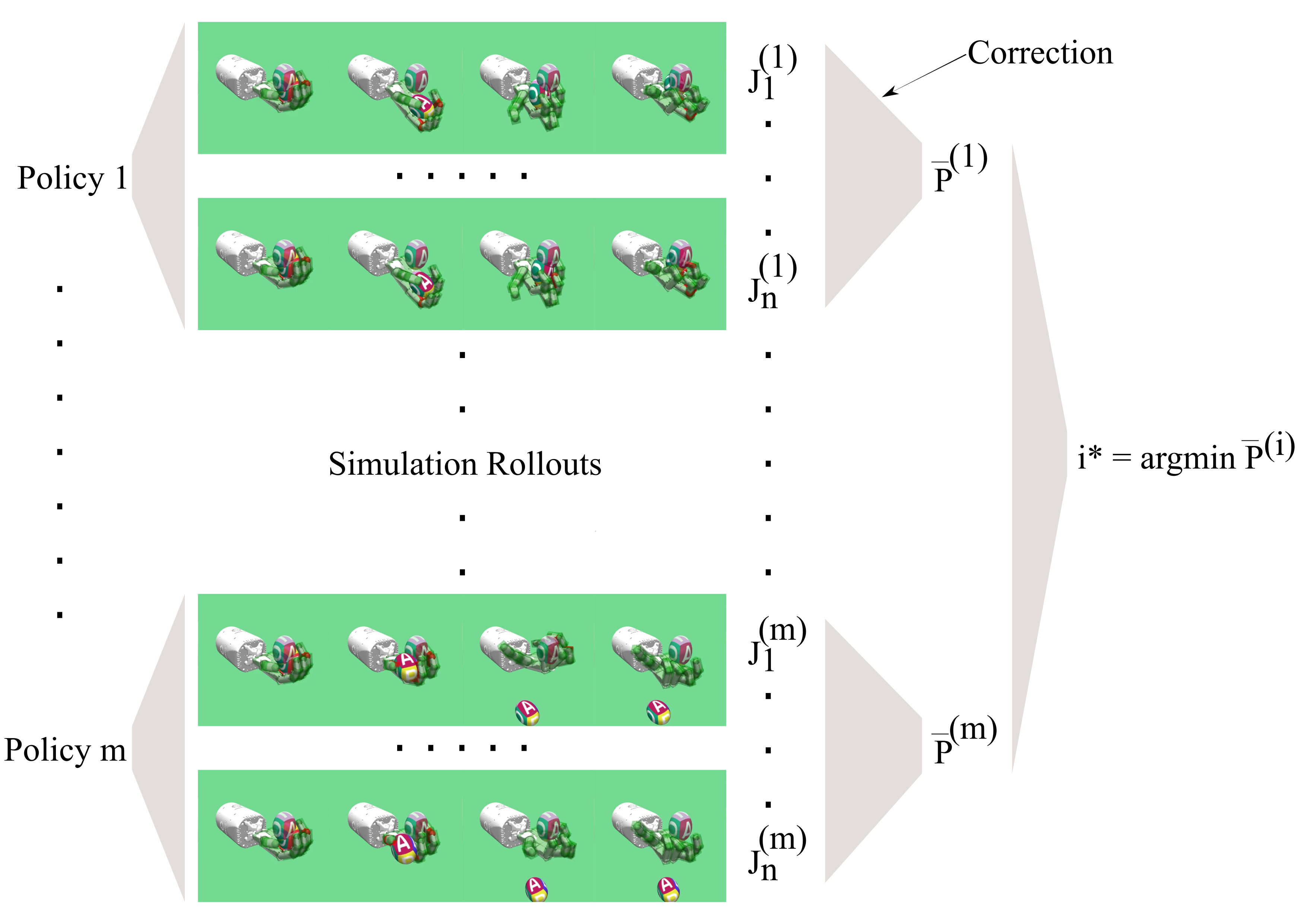}
    \caption{Overview of our method for comparing performance among a set of policies for manipulating an egg of uncertain mass and friction with the MuJoCo Shadow Hand. Given $m$ candidate policies, each policy is executed in simulation $n$ times and each associated trajectory cost is recorded. These cost samples are then used to compute a probabilistic upper bound on performance for each policy $\overline{P}^{(i)}$. Lastly, the policy achieving the minimum bound is selected, giving a probabilistic guarantee on policy performance. To ensure that this final bound holds with a user-specified probability $(1-\delta)$, we apply a multi-hypothesis correction to each of the individual bounds.}
    \label{fig:plan_selection}
\end{figure}

%%%%%%%%%%%%%%%%%%%%%%%%%%%%%%%%%%%%%%%%

%%%%% Manipulation Example %%%%%%%%%%

\section{Manipulation Example} \label{sec:examples}\
In this section we demonstrate the validity and necessity of the multi-hypothesis correction in the MuJoCo Shadow Dexterous Hand environment~\cite{shadow_hand} for simulating in-hand manipulation. In realistic settings, the object to be manipulated may have uncertain physical parameters such as mass and friction that can only be roughly estimated or inferred from interaction but are not directly observable. In this example, we show how for such a setting we can use our distribution-free method to select among several candidate control policies for manipulating the object, obtaining an associated performance bound for the chosen policy. This experiment is chosen to emphasize that our approach works with complex dynamics involving contact and the discontinuities therein.

To generate candidate policies, we use a simple sampling-based planner inspired by the approach in~\cite{howell2022predictive}. Critical to the success of the planner is not to sample actions for each step of the planning horizon, but to take a spline interpolation of sub-sampled points This reduces the effective size of the action space and enforces smoothness. Using this method, we generated $m = 20$ open-loop plans over a horizon of $100$ timesteps each designed for manipulating Gymnasium's egg object assuming a nominal density of $1000 \frac{kg}{m^3}$, sliding friction coefficient of $1$, torsional friction coefficient of $0.005$, and rolling friction coefficient of $0.0001$. We randomize over density rather than mass as this is more standard in MuJoCo, and the object volume is fixed, meaning that randomizing over density and mass are equivalent in this setting.

Given these candidate policies, we applied Theorem~\ref{thm: cor_cov} to inflate the confidence level and select the bound-minimizing policy. We used $\VaR_\tau$ with $\tau = 0.7$ as the cost performance measure we wish to bound and specified a bound error rate of $\delta = 0.2$. The policies are evaluated in simulated environments now randomizing the friction and density of the egg object to simulate uncertainty in the true object's physical parameters. The specific uncertainties are
\begin{itemize}
    \item $\textup{density} \sim \textup{Uniform}[700, 1200] \frac{kg}{m^3}$
    \item $\textup{sliding friction coefficient} \sim \textup{Uniform}[0.8, 1.2]$
    \item $\textup{torsional friction coefficient} \sim \textup{Uniform}[0.004, 0.006]$
    \item $\textup{rolling friction coefficient} \sim \textup{Uniform}[0.00008, 0.00014]$.
\end{itemize}
The individual bound for each policy is computed using the total cost obtained in $30$ simulated executions of it. The cost is based on aligning the egg with a desired goal position and orientation. The chosen goal orientation requires flipping the egg object from its starting orientation. Fig.~\ref{fig:plan_selection} provides an overview of the policy selection and bound computation procedure used in this example and shows selected frames from simulations of the first and last of the 20 plans compared.

\begin{figure}
    \centering
    % \includesvg[width=\linewidth]{figs/valid_fixed_multi_hyp.svg}
    \includegraphics[width=\linewidth]{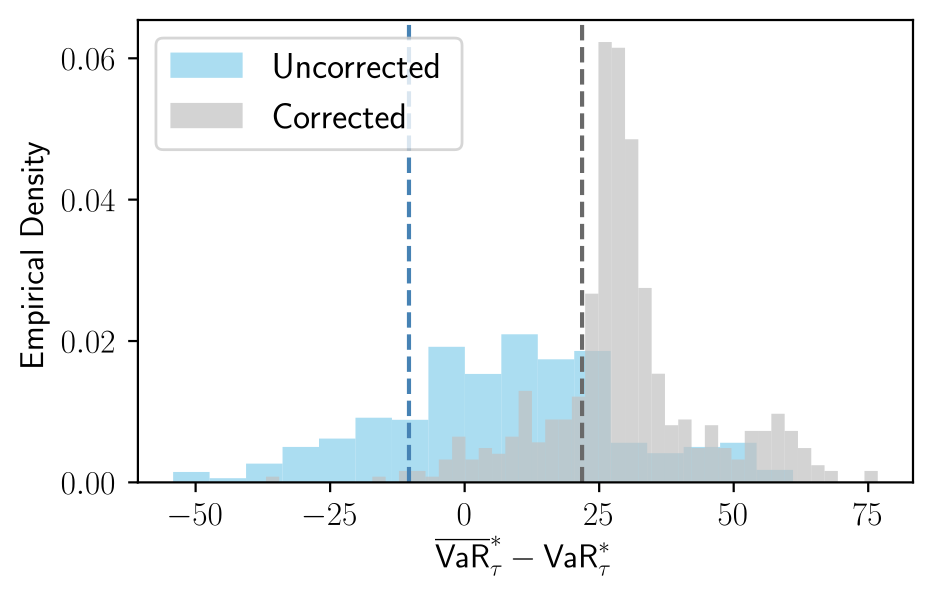}
    \caption{This figure shows the validity of the correction given by Theorem~\ref{thm: cor_cov} for the Shadow Hand environment when manipulating the egg object with uncertain density and friction. Plotted are two distributions, of $\overline{\VaR}_{\tau}^* - \textup{VaR}_\tau^*$. In blue $\overline{\VaR}_{\tau}^*$ is chosen without the correction and in gray $\overline{\VaR}_{\tau}^*$ is chosen with the multi-hypothesis correction specified in Theorem~\ref{thm: cor_cov}. Each sample in the histogram is generated by selecting the best \textcolor{black}{bound ($\overline{\VaR}_{\tau}^*$}) among 20 pre-computed plans based on either the corrected or uncorrected bound \textcolor{black}{and subtracting the true performance measure $\textup{VaR}_\tau^*$ associated with the chosen policy.} Each histogram is generated with 500 repetitions. Only for the corrected bound does the $\delta$ quantile dashed line lie above $0$. Thus when selecting a policy with multi-hypothesis correction the desired error rate is achieved, while this is not the case when using the uncorrected bound.}
    \label{fig:multi_hyp_valid}
\end{figure}

By repeatedly generating fresh cost samples to run many simulated experiments, we show in Fig.~\ref{fig:multi_hyp_valid} that while our multi-hypothesis corrected bound is valid with at least the specified probability of $1 - \delta$, naively using the uncorrected bound is too optimistic and fails to achieve this confidece level. \textcolor{black}{Specifically, we plot the distribution of $\overline{\VaR}_{\tau}^* - \textup{VaR}_\tau^*$ over 500 repetitions where in the blue histogram $\overline{\VaR}_{\tau}^*$ is chosen without correction and in the gray histogram chosen with the multi-hypothesis correction. Each histogram sample is generated by selecting the lowest bound ($\overline{\VaR}_{\tau}^*$) among 20 pre-computed open-loop policies and subtracting the true performance measure $\textup{VaR}_\tau^*$ associated with the chosen policy. Since the bound holds when $\overline{\VaR}_{\tau}^* - \textup{VaR}_\tau^* \geq 0$, if the bound holds with probability at least $1 - \delta$, the $\delta$ quantile of the corresponding $\overline{\VaR}_{\tau}^* - \textup{VaR}_\tau^*$ distribution should be nonnegative. We observe this to be the case when the bound is generated with correction, as the gray dashed line showing the $\delta$ quantile of the associated $\overline{\VaR}_{\tau}^* - \textup{VaR}_\tau^*$ distribution lies right of 0. However, the $\delta$ quantile when generating the bound without correction, shown as the blue dashed line, lies left of 0. Thus, we observe that the corrected bound holds with probability at least $1 - \delta$ as guaranteed by Theorem~\ref{thm: cor_cov} while the uncorrected bound does not, illustrating the necessity of the multi-hypothesis correction.}

\section{Conclusion} \label{sec:conclusion}
In this paper we demonstrate how sampling-based, distribution-free bounds can be used to rigorously bound the performance of a control policy applied in a stochastic environment. These bounds can also be used to verify safety of a policy via constraint tests with a guaranteed false acceptance rate. \textcolor{black}{Furthermore, we provide a thorough analysis of the sensitivity of our bounds to sim-to-real distribution shifts and provide results for constructing robust bounds that can tolerate specified amounts of distribution shift.}
Finally, we show how to apply these bounds when selecting the best policy from a set of candidates which requires a multi-hypothesis correction to retain validity. Because these bounds are distribution-free, they can be applied to complex systems and uncertain environments without requiring knowledge of the underlying problem structure. Rather, our approach only requires simulating policy execution in the stochastic environment and recording the associated cost or constraint value. We empirically demonstrated bound validity in several MuJoCo environments including for the problem of object manipulation which is high-dimensional and has discontinuous dynamics.

In this work, we only studied a few performance measures but our approach extends to other measures if corresponding bounds are available. Another interesting direction for future work is to use our method to select the best risk-sensitive plans in domain-randomized simulations that are then used as training data for a policy with hopes of better sim-to-real performance. \textcolor{black}{Yet another avenue for future work is to use the sample-based bounds for risk and constraint satisfaction to guide an importance sampling procedure (over open-loop plans) within a stochastic model-based planner such as MPPI~\cite{williams2016aggressive} or CEM~\cite{mannor2003cem}.}

%%%%%%%%%%%%%%%%%%%%%%%%%%%%%%%%%%%%%%%%

\section*{Acknowledgments}
The authors would like to acknowledge Taylor Howell for his advice on the sampling-based manipulation policy.

%% Use plainnat to work nicely with natbib. 
\bibliographystyle{ieeetr} % Format as IEEE Transactions
\bibliography{references}

% \newpage

%%%%% Appendix %%%%%%%%%%
\section*{Appendix - Simulation Details}
To create Figure~\ref{fig:valid_stat}, we perturbed the default starting state in MuJoCo using $\textup{reset\_noise\_scale} = 0.1$. The sole exception to this was for the sparse cost case with the Ant where we used $\textup{reset\_noise\_scale} = 0.3$ since the default of 0.1 was never enough to push the Ant outside the healthy range of $[0.2, 1]$ when using the action sequence optimized by CEM.

To identify the candidate policy, cross entropy method (CEM)~\cite{mannor2003cem} was used. 10 generations of optimization were performed, where in each generation 100 open-loop policies were generated by sampling the control input at each timestep using a Gaussian distribution based on the estimated mean and variance from the previous generation. Based on a single execution in the random environment, the top performing 10 sample plans were then used to fit the Gaussian distribution for the next generation. After the final generation, the top performing plan was selected as the candidate. Plans had a horizon of 20 timesteps.

To approximate the distribution over cost, the candidate policy was executed 10,000 times. The associated theoretical statistic was then found by taking the empirical average cost for $\mathbb{E}$, the empirical quantile for $\textup{VaR}_{\tau}$, or the Monte Carlo approximation for $\textup{CVaR}_{\tau}$ (see~\cite{shapiro2021lectures}). To compute $\mathbb{E}$ and $\textup{CVaR}_{\tau}$, the total costs in the Ant environment were clipped to between $[-2H, 0]$ and between $[-0.325H, 0.1H]$ for the Swimmer environment where $H =~20$ was the horizon.

%%%%%%%%%%%%%%%%%%%%%%%%%%%%%%%%%%%%%%%%%%%%%%%%%%%%%%%%%%%%
\section*{Appendix - Bound Comparison}
\textcolor{black}{Here we compare the bounds we use for $\VaR_{\tau}$, expected value, and $\CVaR_{\tau}$ with those used by other papers in the robotics and statistics literature (\cite{,caltech_policy_synth,risk_verification,cvar_kolla}). Figure~\ref{fig:compare_bounds} shows the empirical bound distribution generated by each approach using the the same experimental parameters as when constructing Figure~\ref{fig:valid_stat}. In every case the bounds we use are the least conservative (better estimating the unknown performance measure) while still meeting the desired $1-\delta$ confidence level.}

\textcolor{black}{From the derivation of our $\VaR_{\tau}$ bound, it is clear that choosing a smaller order statistic for the bound provably violates the $1-\delta$ confidence level. Thus, amongst $\VaR_{\tau}$ bounds constructed as a particular order statistic, ours is unimprovable. Because of this property, our $\VaR_{\tau}$ bound will always be less conservative than those used in~\cite{risk_verification,cvar_kolla}. Statements of similar generality cannot be made in the case of the expected value and $\CVaR_{\tau}$ bounds, although the empirical results in Figure~\ref{fig:compare_bounds} provides good evidence that the bounds we use in these cases are to be preferred to the bounds used in~\cite{caltech_policy_synth,risk_verification,cvar_kolla}.}

\begin{figure}
     \centering
     % \subfloat[]{\includesvg[width=\linewidth]{figs/compare_fixed_var.svg}\label{fig: compare_var_bound}}
     \subfloat[]{\includegraphics[width=\linewidth]{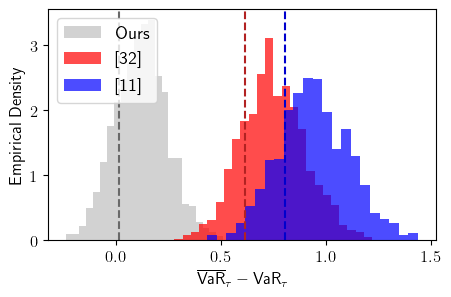}\label{fig: compare_var_bound}}
     \hfill
     % \subfloat[]{\includesvg[width=\linewidth]{figs/compare_fixed_exp.svg}\label{fig: compare_exp_bound}}
     \subfloat[]{\includegraphics[width=\linewidth]{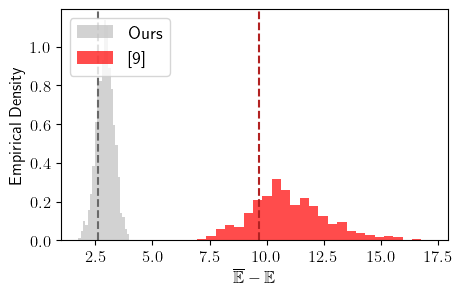}\label{fig: compare_exp_bound}}
     \hfill
     % \subfloat[]{\includesvg[width=\linewidth]{figs/compare_fixed_cvar.svg}\label{fig: compare_cvar_bound}}
     \subfloat[]{\includegraphics[width=\linewidth]{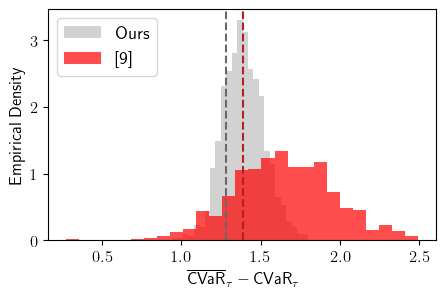}\label{fig: compare_cvar_bound}}
    \caption{\textcolor{black}{Distribution of bound offsets for (a) $\VaR_{\tau}$ in the Half Cheetah environment, (b) expectation in the Ant environment, (c) $\CVaR_{\tau}$ in the Swimmer environment. All bounds are probabilistically valid, since the vertical dashed lines (the $\delta$ quantiles) are to the right of zero, indicating that the generated bounds indeed exceed the true performance measure with probability at least 1-$\delta$. In every case the bounds we use are less conservative than the bounds used in \cite{caltech_policy_synth,risk_verification,cvar_kolla}. As in Fig.~\ref{fig:valid_stat}, the true performance measure is computed using $10,000$ rollouts and each bound histogram was generated by repeatedly computing the given sample-based bound $1000$ times, each time using a fresh set of $n = 100$ sampled policy rollouts. We again use $\tau = 0.7$.}}
    \label{fig:compare_bounds}
    
\end{figure}

%%%%%%%%%%%%%%%%%%%%%%%%%%%%%%%%%%%%%%%%%%%%%%%%%%%%%%%%%%%%
\section*{Appendix - Proofs}
\textcolor{black}{Although we don't consider lower bounds or two-sided bounds, we note that lower bounds for each performance measure can be derived in similar fashion to our derivations for the upper bounds. Two-sided bounds can then be constructed by combining lower and upper bounds and adjusting the confidence level (requiring each one-sided bound to hold with probability $1-\delta/2$) using the inclusion-exclusion principle.}

\subsection{Proof of Theorem \ref{thm: var_bound}}
\textcolor{black}{We first provide a lemma that is necessary for our proof.}
% \subsubsection{Upper Bound}
\begin{lemma} \label{lemma}
    \begin{gather}
        \Pr[J_i < \VaR_\tau(J)] \leq \tau
    \end{gather}
\end{lemma}
\begin{proof}
We may rewrite the strict inequality probability using a left limit approaching $\VaR_\tau(J)$:
\begin{gather}
    \Pr[J_i < \VaR_\tau(J)] = \lim_{x \rightarrow \VaR_\tau(J)^{-}} \Pr[J_i \leq x]
\end{gather}
and are guaranteed that the limit exists since CDFs are upper semicontinuous. Since $x$ in the limit satisfies $x < \VaR_\tau(J)$, by definition of $\VaR$ as an infimum we have
\begin{gather}
    \Pr[J_i \leq x] < \tau.
\end{gather}
Since this holds for all $x$ in the limit,
\begin{gather}
    \lim_{x \rightarrow \VaR_\tau(J)^{-}} \Pr[J_i \leq x] \leq \tau
\end{gather}
concluding the proof.
\end{proof}

Now, to prove the main result, from the definition of $\VaR_\tau(J)$ we have
    \begin{subequations}
        \begin{gather}
        \Pr[\VaR_\tau(J) \le J_{(k)}] = \Pr[\sum_{i=1}^n \mathbbm{1}(J_i < \VaR_\tau(J)) < k] \\
        = \Pr[\sum_{i=1}^n \mathbbm{1}(J_i < \VaR_\tau(J)) \le k-1].
    \end{gather}
    \end{subequations}
Since by Lemma \ref{lemma} we have
    \begin{gather}
        \Pr[J_i < \VaR_\tau(J)] \leq \tau \label{eq: relaxation},
    \end{gather}
 the random quantity $\mathbbm{1}(J_i < \VaR_\tau(J))$ is Bernoulli where the probability of being $1$ is at most $\tau$. Thus,
    \begin{align}
        \Pr[\sum_{i=1}^n \mathbbm{1}(J_i < \VaR_\tau(J)) \le k-1] \geq \textup{Bin}(k-1;n,\tau) \label{eq: binomial}
    \end{align}
since we have the sum of $n$ IID Bernoulli random variables with probability of being 1 at most $\tau$. 

Therefore, given values for $\tau, \delta \in (0,1)$ and $J_{1:n}$, choosing
\begin{subequations}
    \begin{gather}
    \overline{\VaR}_{\tau} = J_{(k^*)}, \\
    k^* = \min \{k \mid \textup{Bin}(k-1;n,\tau) \ge 1 - \delta\} \label{eq: var_bound}
    \end{gather}
\end{subequations}
ensures that
\begin{gather}
    \Pr[\VaR_{\tau}(J) \le \overline{\VaR}_{\tau}] \geq 1-\delta.
\end{gather}
% Lastly, note that 
% \begin{gather} \label{eq: relaxation2}
%     \Pr[\VaR_\tau(J) \leq \overline{\VaR}_\tau] \geq \Pr[\VaR_{\tau}(J) < \overline{\VaR}_{\tau}]
% \end{gather}
% so we obtain the desired result.
\begin{flushright}
    $\blacksquare$
\end{flushright}

\begin{remark}
Note that in the case where $J$ has invertible CDF, (\ref{eq: relaxation}) is tight so that (\ref{eq: binomial}) holds exactly. In fact, we can then give a more precise result in this case,
\begin{align}
    1-\delta \le \Pr[ \textup{VaR}_\tau(J) \leq \overline{\VaR}_\tau] \le 1-\delta + \textup{bin}(k-1;n,\tau)
\end{align}
where $k$ is the order statistic chosen for $\overline{\VaR}_\tau$ and $\textup{bin}(k-1;n,\tau)$ denotes the Binomial probability mass function evaluated at $k-1$ successes. Thus, the amount of conservatism in the coverage is no more than $\textup{bin}(k-1;n,\tau)$. 

% As an example, for $\tau = 0.7$, $1 - \delta = 0.9$, $n=50$, we find $k = 40$ and then we have
% \begin{align}
%     0.9 \le \Pr[ \textup{VaR}_\tau(Y) < \overline{\textup{VaR}}_{\tau,\delta} ] = 0.92 \le 0.9 + 0.06
% \end{align}

\end{remark}

\begin{remark}
Alternatively, one can achieve exactly the desired confidence by randomizing the chosen order statistic~\cite{zielinski2005best}.
\end{remark}

\begin{remark}
Our derivation of the $\VaR$ bound adapts a very similar result from conformal prediction presented below~\cite{hulsman2022distribution}.

\begin{theorem}[Conformal Prediction Conditional Result] \label{thm: conditional}
    Given $n+1$ IID samples $S_{1:n+1}$ from some continuous distribution,
        \begin{multline}
            \Pr\Big[\Pr[S_{n+1} \leq S_{(k)} \mid S_{1:n} ] \geq 1 - \epsilon\Big] = \\ \textup{Bin}(k-1;n,1-\epsilon).
        \end{multline}
\end{theorem}
In fact, our proof does not assume a continuous distribution and can be applied to adapt the above result to $\geq \Bin(k-1;n,1-\epsilon)$ for any distribution.
\end{remark}

\subsection{Proof of Theorem \ref{thm: exp_bound}}
% \subsubsection{Upper Bound}
The result follows as a special case of Theorem~\ref{thm: cvar_bound} proven below. We let $\tau = 0$, noting that $\overline{\CVaR}_\tau = \overline{\E}$ when $\tau = 0$ from the definitions of $\CVaR_\tau$ and $\E$ in Def.~\ref{Def:CVaR} and Def.~\ref{Def:E}.
\begin{flushright}
    $\blacksquare$
\end{flushright}

% \subsubsection{2-Sided Bound}
% Similarly, the 2-sided bound is obtained from the 2-sided $\CVaR_\tau$ bound in the proof of Theorem~\ref{thm: cvar_bound} below by setting $\tau = 0$.  We obtain the lower bound
% \[
% \underline{\E} =
% \epsilon y_\lb + \left(1 - \frac{k_u - 1}{n} - \epsilon\right)Y_{(k_u)} + \frac{1}{n}\sum_{i = 1}^{k_u-1}Y_{(i)},
% \]
% which holds with probability $1-\delta$. Applying the bounds simultaneously using the inclusion-exclusion principle as in the CVaR case gives
% \[
% \Pr[\underline{\E}\le \E[Y] \le \overline{\E}] \ge (1-2\delta).
% \]
% \begin{flushright}
%     $\blacksquare$
% \end{flushright}

\subsection{Proof of Theorem \ref{thm: cvar_bound}}
% \subsubsection{Upper Bound}
To construct an upper bound for CVaR we first revisit the definition,
\begin{align}        
    \CVaR_\tau(Y) \coloneqq \frac{1}{1-\tau}\int_\tau^1\VaR_\gamma(Y)\,d\gamma.
\end{align}
Note that if we can construct an upper bound on the VaR that holds for all $\gamma$, then we can use this to find an upper bound on the CVaR,
\begin{align}        
    \overline{\CVaR}_\tau(Y) = \frac{1}{1-\tau}\int_\tau^1 \overline{\VaR}_\gamma(Y)\,d\gamma.
\end{align}
Note the slight abuse of notation here, in this proof we require a \textit{simultaneous} VaR bound for which 
\begin{equation}
    \Pr[\overline{\VaR}_\tau \ge \VaR_\tau \ \forall \tau] \ge 1-\delta,
\end{equation}
a stronger requirement than we had for the VaR bound presented in Theorem \ref{thm: var_bound}. Next we describe how we obtain a simultaneous VaR bound.

Consider IID samples $Y_1, \ldots, Y_n$ with unknown distribution given by $\CDF(y)$ and let $y_{ub}$ be an associated almost sure upper bound.~\footnote{We replace our $J$ notation with $Y$ to avoid confusion of $j$ as an index.} Let $\widehat{\CDF}(y)$ be the empirical CDF, 
\begin{gather}
\widehat{\CDF}(y) = \frac{1}{n}\sum_{i = 1}^n\mathbbm{1}(Y_i \le y ),
\end{gather}
and $\epsilon(n,\delta)$ is a constant subtracted from $\widehat{\CDF}(y)$ to obtain a probabilistic lower bound
\begin{align}
    \underbar{\CDF}(y) = \begin{cases}
        \max\{\widehat{\CDF}(y)-\epsilon, 0\} &\text{   if } y<y_{\ub} \\
        \underbar{\CDF}(y) = 1 &\text{   if }y \ge y_{\ub}.
    \end{cases}
\end{align}
By letting
\begin{equation}
\epsilon(n,\delta) = \sqrt{\frac{-\ln \delta}{2n}},
\end{equation} 
(what we call the DKW gap in Def.~\ref{Def:DKWGap}) we know from the DKW bound~\cite{dvoretzky1956asymptotic,massart1990tight} that 
\begin{equation}
    \Pr[\underbar{CDF}(y) \le \CDF(y) \ \forall y] \ge 1 - \delta.
\end{equation} 
Let $\overline{\VaR}_\tau$ be the $\VaR_\tau$ obtained from $\underline{\CDF}(y)$, and note that the DKW bound extends to $\overline{\VaR}_\tau$ to give 
\begin{equation}
\Pr[\overline{\VaR}_\tau \ge \VaR_\tau \ \forall \tau] \ge 1 - \delta.
\end{equation}

To see this, observe that $\underline{\CDF}(y) \le \CDF(y)$ for all $y$ implies $\{y \mid \underline{\CDF}(y) \ge \tau\} \subseteq \{y \mid \CDF(y) \ge \tau\}$, therefore $\inf\{y \mid \underline{\CDF}(y) \ge \tau\} \ge \inf\{y \mid \CDF(y) \ge \tau\}$, since the inf over a subset is greater than or equal to the inf over the larger set.  We have $\underline{\CDF}(y) \le \CDF(y)$ for all $y$ implies $\overline{\VaR}_\tau \ge \VaR_\tau$ for all $\tau$ (from the definition of $\VaR_\tau$ in Def.~\ref{Def:VaR}). The extension of the DKW bound to $\overline{\VaR}_\tau$ follows.

From this bound, we conclude that integrating $\overline{\VaR}_\tau$ over any $\tau$ interval gives an upper bound on the integral of $\VaR_\tau$ over the same interval, which holds with probability at least $1-\delta$.  We proceed to analytically integrate $\overline{\VaR}_\tau$ from $\tau$ to 1 to compute $\overline{\CVaR}_\tau$ based on the definition of $\CVaR_\tau$ in Def.~\ref{Def:CVaR}.

Note that $\overline{\VaR}_\tau$ is a staircase function defined on the domain $\tau \in [0, 1]$, which is equal to the smallest order statistic $Y_{(k)}$ such that $(\frac{k}{n} - \epsilon) \ge 0$ over the interval $\tau \in [0, (\frac{k}{n} - \epsilon)]$. It is then equal to $Y_{(k+1)}$ over the next interval $\tau \in ((\frac{k}{n} - \epsilon), (\frac{k+1}{n} - \epsilon)]$, proceeding to $Y_{(n)}$ over $\tau \in ((\frac{n-1}{n} - \epsilon), (1 - \epsilon)]$, and finally $y_\ub$ over the last interval of $\tau \in ((1 - \epsilon), 1]$. Notice that all intervals are of length $\frac{1}{n}$, except for the first, which is of length $(\frac{k}{n} - \epsilon)$, and the last, which is of length $\epsilon$.  

Integrating this staircase function from a given $\tau$ to 1, therefore, evaluates to a sum over order statistics times the length of their respective intervals. The first order statistic in this sum is the smallest such that its interval appears above the $\tau$ quantile, namely $Y_{(k)}$, where $k$ is the smallest index such that $(\frac{k}{n} - \epsilon) \ge \tau$. The length of this first interval is then $(\frac{k}{n} - \epsilon - \tau)$, and we have for the first term in the sum $(\frac{k}{n} - \epsilon - \tau)Y_{(k)}$, followed by $n-k$ terms of the form $\frac{1}{n}Y_{(i)}$, where $i = k+1, \ldots, n$, and finally the term $\epsilon y_\ub$. Following the definition of $\CVaR_\tau$ in Def.~\ref{Def:CVaR}, we normalize the sum by the length of the interval over which we integrate, $\frac{1}{1-\tau}$, to obtain the desired expression.

The bound holds for any $\tau \in [0, 1)$ and $\delta \in (0, 0.5]$ (this requirement comes from the DKW inequality). To avoid defaulting to $y_{ub}$ we require that there is an index $k \leq n$ such that $(\frac{k}{n} - \epsilon - \tau) \geq 0$. Equivalently, $\epsilon \leq 1 - \tau$ which implies $n \geq -\frac{1}{2}\ln(\delta)/(1-\tau)^2$.

As noted earlier, this $\CVaR$ bound is mathematically equivalent to the one in~\cite{cvar_thomas}, but our derivation results in a different form for the bound expression. Visualizations of the integral form of CVaR are given in Figures 3 and 4 of~\cite{cvar_thomas}.

\begin{flushright}
    $\blacksquare$
\end{flushright}

\subsection{Proof of Theorem \ref{thm: feasibility_bound}}
\textcolor{black}{
In this proof we show that $\Pr[q \le \overline{q}] \ge 1-\delta$ for $\overline{q}$ chosen as described in Theorem \ref{thm: feasibility_bound}. By the definition of $\overline{q}$, note that 
\begin{gather}
    q \le \overline{q} \iff \textup{Bin}(k;n,q) \ge \delta \iff k \ge k^*,
\end{gather}
where 
\begin{gather}
    k^* = \min \{k' \in \{0,\ldots,n\} \mid \textup{Bin}(k';n,q) \ge \delta \}.
\end{gather}
Then
\begin{subequations}
    \begin{align}
    \Pr[q \le \overline{q}] &= \Pr(k \ge k^*) = 1 - \Pr(k \le k^* -1), \\
    &= 1 - \textup{Bin}(k^*-1;n,q) \ge 1-\delta.
\end{align}
\end{subequations}
We know $\textup{Bin}(k^*-1;n,q) < \delta$ by construction of $k^*$. Note that to construct the bound we do not need knowledge of $q$, whatever $q$ might be, our rule for choosing $\overline{q}$ is valid.}
\begin{flushright}
    $\blacksquare$
\end{flushright}

\begin{remark}
\textcolor{black}{
This is the one-sided Clopper-Pearson bound~\cite{clopper1934use} which is known to be unimprovable amongst non-randomized approaches~\cite{wang_clopper}. Note that one could also arrive at this result by employing our VaR bound (Theorem \ref{thm: var_bound}). Going this direction shows that applying conformal prediction to samples from a Bernoulli distribution reduces to the Clopper-Pearson result.}
% Although we derived $\tau^*$ using our earlier VaR bound in Theorem \ref{thm: var_bound}, (\ref{eq: prob_lower_bound}) takes the same form as the one-sided Clopper-Pearson lower bound~\cite{clopper1934use,wang_clopper} which is known to be unimprovable amongst non-randomized approaches~\cite{wang_clopper}.
\end{remark}

\subsection{Proof of Theorem \ref{Thm:ConstraintTest}}
Let $\overline{P}$ be a finite-sample bound for performance measure $\P(g)$ with error rate $\delta$:
\begin{gather}
    \Pr[\P(g) \leq \overline{P}] \geq 1-\delta
\end{gather}
and consider the constraint 
\begin{gather}
    \P(g) \leq C.
\end{gather}

For the test which accepts when $\overline{P} \leq C$ we can upper bound the false acceptance rate as
\begin{align}
\Pr[ \overline{P} \le C \mid \P(g) > C] \leq \Pr[\overline{P} < \P(g)].
\end{align}
Then, by the guaranteed error rate of $\delta$, we get
\begin{align}
\Pr[\overline{P} < \P(g)] = 1 - \Pr[\overline{P} \ge \P(g)] \leq \delta.
\end{align}
Combining, we obtain the desired bound on the test's false acceptance rate
\begin{align}
    \Pr[ \overline{P} \le C \mid \P(g) > C] \leq \delta.
\end{align}
\begin{flushright}
    $\blacksquare$
\end{flushright}

\subsection{Proof of Theorem \ref{thm: uncor_cov}}
Suppose we are given $m$ policies $\mathcal{U}_{1:m}$, with  unknown true performance $\{P^{(i)}\}_{i=1}^m$ and associated probabilistic bounds $\{\overline{P}^{(i)}\}_{i=1}^m$ individually holding with probability at least $1-\delta$ i.e., 
\begin{align}
    \Pr[P^{(i)} \leq \overline{P}^{(i)}] \geq 1-\delta.
\end{align}
Let $\overline{P}^*$ be the lowest probabilistic bound and let $P^*$ be the associated true statistic i.e., 
    \begin{subequations}
        \begin{gather}
            i^* = \argmin_{i=1:m} \overline{P}^{(i)}, \\ 
            \overline{P}^* = \overline{P}^{(i^*)}, \\
            P^* = P^{(i^*)}.
        \end{gather}
    \end{subequations}

First note that the probability of the minimum bound holding is at least as likely as all the bounds holding:
\begin{gather}
    \Pr[P^* \leq \overline{P}^*] \geq \Pr[P^{(i)} \leq \overline{P}^{(i)} \ \forall i].
\end{gather}
Note that even though we do not assume $\mathcal{U}_i$ are independent, we can assert the events $P^{(i)} \le \overline{P}^{(i)}$ and $P^{(j)} \le \overline{P}^{(j)}$ are independent for $i \neq j$ as bound $i$ is generated with different random samples than bound $j$. Applying independence yields
\begin{gather}
    \Pr[P^{(i)} \leq \overline{P}^{(i)} \ \forall i] = \prod_{i=1}^m \Pr[P^{(i)} \leq \overline{P}^{(i)}].
\end{gather}
Each bound holds with probability at least $1-\delta$ so
\begin{gather}
    \prod_{i=1}^m \Pr[P^{(i)} \leq \overline{P}^{(i)}] \geq (1-\delta)^m.
\end{gather}
Combining the steps, conclude that
\begin{gather}
    \Pr[P^* \leq \overline{P}^*] \geq (1-\delta)^m.
\end{gather}

\subsection{Proof of Theorem \ref{thm: cor_cov}}
Applying Theorem \ref{thm: uncor_cov} with $\bar{\delta} = 1 - (1-\delta)^{1/m}$ we get
\begin{align}
    \Pr[P^* \leq \overline{P}^*] \geq (1-\bar{\delta})^m = 1 - \delta.
\end{align}
\begin{flushright}
    $\blacksquare$
\end{flushright}

\subsection{\textcolor{black}{Proof of Corollary \ref{cor: sens_var}}}
\textcolor{black}{
Given
\begin{align}
    \sup_x \quad \text{CDF}_{\mathcal{D}_{sim}}(x) - \text{CDF}_{\mathcal{D}_{true}}(x) \le \alpha,
\end{align}
we know
\begin{align}
    \VaR_\tau(J_{true}) \le \VaR_{\tau + \alpha}(J_{sim}).
\end{align}
Furthermore, 
\begin{subequations}
    \begin{gather}
        \overline{\VaR}_\tau = k^*, \\
    k^* = \min \{k \in \{1,\ldots,n\} \mid  \Bin(k-1;n,\tau) \ge 1 - \delta_{sim}\}.
    \end{gather}
\end{subequations}
Then, utilizing the proof for Theorem \ref{thm: var_bound}, we have
\begin{subequations}
    \begin{align}
     \Pr[\VaR_\tau(J_{true}) \le \overline{\VaR}_\tau] &\ge \Pr[\VaR_{\tau + \alpha}(J_{sim}) \le k^*], \\
    &\ge \textup{Bin}(k^* - 1; n, \tau+\alpha), \\
    &\ge 1 - \delta_{true}
    \end{align}
\end{subequations}
}
\begin{flushright}
    $\blacksquare$
\end{flushright}

\subsection{\textcolor{black}{Proof of Corollary \ref{cor: sens_cvar}}}
\textcolor{black}{
We are given
\begin{align}
    \sup_x \ \text{CDF}_{\mathcal{D}_{sim}}(x) - \text{CDF}_{\mathcal{D}_{true}}(x) \le \alpha \le \sqrt{\frac{-\ln(2\delta_{sim})}{2n} }.
\end{align}
From the proof of Theorem \ref{thm: cvar_bound}, the CVaR bound utilizes a simultaneous lower bound over the CDF of cost. We can interpret the distribution shift by amount $\alpha$ as a shift in the offset of our CDF bound from the empirical CDF,
\begin{align}
    \underbar{\CDF}(y) = \begin{cases}
        \max\{\widehat{\CDF}(y)-\epsilon^\prime, 0\} &\text{   if } y<y_{\ub} \\
        \underbar{\CDF}(y) = 1 &\text{   if }y \ge y_{\ub}.
    \end{cases}
\end{align}
where $\epsilon^\prime = \epsilon - \alpha$. From the DKW inequality~\cite{dvoretzky1956asymptotic,massart1990tight}, we know
\begin{subequations}
    \begin{align}
    \epsilon^\prime &= \sqrt{\frac{-\ln \delta_{true}}{2n}}, \\
    \implies \delta_{true} &= e^{-2n\epsilon^{\prime 2}}, \\
    &= e^{-2n(\epsilon - \alpha)^2}.
\end{align}
\end{subequations}
}
\begin{flushright}
    $\blacksquare$
\end{flushright}

\subsection{\textcolor{black}{Proof of Corollary \ref{cor: sens_fail_prob}}}
\textcolor{black}{
Given $k_{sim}$ observed failures in the simulator out of $n$ trajectories, we have
\begin{align}
    \overline{q} = \max \{q' \in [0,1] \mid \textup{Bin}(k_{sim};n,q') \ge \delta_{sim} \}.
\end{align}
Utilizing the proof of Theorem \ref{thm: feasibility_bound}, we have
\begin{subequations}
    \begin{align}
    \Pr[q_{true} \le \overline{q}] &\ge \Pr[q_{sim}+\alpha \le \overline{q}], \\
    &= \Pr[k_{sim} \ge k^*_{\alpha}],
\end{align}
\end{subequations}
where
\begin{align}
    k^*_{\alpha} = \min \{k \in \{0,\ldots,n\} \mid \textup{Bin}(k;n,q_{sim}+\alpha) \ge \delta_{sim} \}.
\end{align}
Thus,
\begin{subequations}
    \begin{align}
    \Pr[q_{true} \le \overline{q}] &\ge 1 - \Pr[k_{sim} \le k^*_{\alpha} -1], \\
    &= 1 - \textup{Bin}(k_{\alpha}^*-1;n,q_{sim}), \\
    &\ge 1-\delta_{true}.
\end{align}
\end{subequations}
}
\begin{flushright}
    $\blacksquare$
\end{flushright}

\subsection{\textcolor{black}{Proof of Corollary \ref{cor: robust_var}}}
\textcolor{black}{
Given 
\begin{align}
    \sup_x \quad \text{CDF}_{\mathcal{D}_{sim}}(x) - \text{CDF}_{\mathcal{D}_{true}}(x) \le \alpha,
\end{align}
we know
\begin{align}
    \VaR_\tau(J_{true}) \le \VaR_{\tau + \alpha}(J_{sim}).
\end{align}
Therefore,
\begin{subequations}
    \begin{align}
    \Pr[\VaR_\tau(J_{true}) \le \overline{\VaR}_{\tau}(\alpha)], \\
    \ge \Pr[\VaR_{\tau + \alpha}(J_{sim}) \le \overline{\VaR}_{\tau + \alpha}], \\
    \ge 1-\delta.
\end{align}
\end{subequations}
}
\begin{flushright}
    $\blacksquare$
\end{flushright}

\subsection{\textcolor{black}{Proof of Corollary \ref{cor: robust_cvar}}}
\textcolor{black}{
Given 
\begin{align}
    \sup_x \quad \text{CDF}_{\mathcal{D}_{sim}}(x) - \text{CDF}_{\mathcal{D}_{true}}(x) \le \alpha,
\end{align}
we can create a lower bound on $\text{CDF}_{\mathcal{D}_{true}}$ by lowering a bound on $\text{CDF}_{\mathcal{D}_{sim}}$ by amount $\alpha$. That is,
\begin{equation}
    \Pr[\underbar{CDF}(y) \le \text{CDF}_{\mathcal{D}_{true}}(y)\  \ \forall y] \ge 1 - \delta
\end{equation} 
where, following the proof of Theorem \ref{thm: cvar_bound}, and using $\epsilon^\prime = \epsilon + \alpha = \sqrt{-\ln \delta / 2n} + \alpha$,
\begin{align}
    \underbar{\CDF}(y) = \begin{cases}
        \max\{\widehat{\CDF}(y)-\epsilon^\prime, 0\} &\text{   if } y<y_{\ub} \\
        \underbar{\CDF}(y) = 1 &\text{   if }y \ge y_{\ub}.
    \end{cases}
\end{align}
Then, following the proof of Theorem \ref{thm: cvar_bound}, the robust bound is constructed by simply replacing $\epsilon$ in the original bound equation with $\epsilon^\prime$.
}
% \begin{align}
%     \CVaR_\tau(J_{true}) \le \frac{1}{1-\tau}\int_\tau^1\VaR_{\gamma + \alpha}(J_{sim}) \,d\gamma.
% \end{align}
\begin{flushright}
    $\blacksquare$
\end{flushright}

\subsection{\textcolor{black}{Proof of Corollary \ref{cor: robust_fail_prob}}}
\textcolor{black}{
Given 
\begin{align}
    \sup_x \quad \text{CDF}_{\mathcal{D}_{sim}}(x) - \text{CDF}_{\mathcal{D}_{true}}(x) \le \alpha,
\end{align}
we know $q_{true} \le q_{sim} + \alpha$. Then let
\begin{align}
    k^* = \min \{k' \in \{0,\ldots,n\} \mid \textup{Bin}(k';n,q_{sim}) \ge \delta \}.
\end{align}
With the following implications,
\begin{subequations}
    \begin{align}
    k \ge k^* \implies \Bin(k;n,q_{sim}) \ge \delta, \\
    \implies \Bin(k;n,q_{true} - \alpha) \ge \delta \implies q_{true} \le \overline{q}(\alpha),
\end{align}
\end{subequations}
we know
\begin{subequations}
    \begin{align}
    \Pr[q_{true} \le \overline{q}(\alpha)] \ge \Pr[k \ge k^*], \\
    = 1 - \Pr[k \le k^*-1] \ge 1 - \delta.
\end{align}
\end{subequations}
}
\begin{flushright}
    $\blacksquare$
\end{flushright}

\end{document}